\documentclass{article}

\usepackage[preprint]{neurips_2026}

\usepackage[utf8]{inputenc} 
\usepackage[T1]{fontenc}    
\usepackage{hyperref}       
\usepackage{url}            
\usepackage{booktabs}       
\usepackage{amsfonts}       
\usepackage{nicefrac}       
\usepackage{microtype}      
\usepackage{xcolor}         

\usepackage{natbib}
\usepackage{caption}
\usepackage{graphicx}
\usepackage{amsmath}
\usepackage{amsthm}
\usepackage{subcaption}
\usepackage{amssymb}
\usepackage{mathtools}
\usepackage{bm}
\usepackage{xspace}
\usepackage[inline]{enumitem}
\usepackage{threeparttable}
\usepackage{tikz}
\usepackage{multirow}
\usepackage{multicol}
\usepackage{comment}
\usepackage{tabularx}
\newcolumntype{Y}{>{\raggedright\arraybackslash}X}

\title{Scalable Derivative Gaussian Processes via Exact Gradient Reduction}

\author{%
  Hyunseok Seung \\
  Department of Statistics \\
  University of Wisconsin--Madison \\
  Madison, WI, 53706 \\
  \texttt{hseung2@wisc.edu} \\
  \And
  Matthias Katzfuss \\
  Department of Statistics \\
  University of Wisconsin--Madison \\
  Madison, WI, 53706 \\
  \texttt{katzfuss@gmail.com}
}

\usepackage{graphicx,bm,colonequals,amsmath,amssymb,url,xcolor,bbm}
\usepackage{array,tabularx,multirow}
\usepackage{soul} 
\usepackage{placeins} 
\usepackage[normalem]{ulem}

\graphicspath{{plots/}}

\newcommand{\todo}[1]{\textcolor{red}{[#1]}}

\usepackage{mathtools}

\usepackage{amsthm}

\newtheorem{proposition}{Proposition}

\newtheorem{lemma}{Lemma}

\newtheorem{remark}{Remark}

\usepackage[ruled,vlined]{algorithm2e}
\SetKwInput{KwInput}{Input}
\SetKwInput{KwOutput}{Output}
\usepackage{algorithmic}
\usepackage{framed}
\usepackage{float}

\newcommand{\bh}{\mathbf{h}}
\newcommand{\bu}{\mathbf{u}}

\newcommand{\ba}{\mathbf{a}}
\newcommand{\bv}{\mathbf{v}}

\newcommand{\bd}{\mathbf{d}}

\newcommand{\bx}{\mathbf{x}}
\newcommand{\by}{\mathbf{y}}
\newcommand{\bq}{\mathbf{q}}
\newcommand{\bg}{\mathbf{g}}
\newcommand{\bk}{\mathbf{k}}
\newcommand{\br}{\mathbf{r}}

\newcommand{\be}{\mathbf{e}}

\newcommand{\bF}{\mathbf{F}}

\newcommand{\bI}{\mathbf{I}}
\newcommand{\bD}{\mathbf{D}}
\newcommand{\bH}{\mathbf{H}}

\newcommand{\bK}{\mathbf{K}}
\newcommand{\bX}{\mathbf{X}}

\newcommand{\bR}{\mathbf{R}}

\newcommand{\bfzero}{\mathbf{0}}

\newcommand{\bfmu}{\bm{\mu}}

\newcommand{\bftheta}{\bm{\theta}}

\newcommand{\bfepsilon}{\bm{\epsilon}}

\newcommand{\bfLambda}{\bm{\Lambda}}
\newcommand{\bfDelta}{\bm{\Delta}}

\newcommand{\cov}{\textbf{Cov}}

\newcommand{\GP}{\mathcal{GP}}

\newcommand{\order}{\mathcal{O}}
\newcommand{\normal}{\mathcal{N}}

\DeclareMathOperator*{\KL}{KL}

\DeclareMathOperator*{\bbE}{\mathbb{E}}
\newcommand{\bbR}{\mathbb{R}}

\newcommand{\domain}{\mathcal{D}}

\newcommand{\tera}[0]{\textsc{Tera}\xspace}

\begin{document}

\maketitle

\begin{abstract}
Gradient observations can substantially improve Gaussian process (GP) surrogates, particularly in high-dimensional settings where function evaluations are expensive. However, exact inference with $n$ function values and $n$ full gradients in $d$ dimensions scales cubically in the joint state size, imposing an intractable $\mathcal{O}(n^3 d^3)$ computational bottleneck. We introduce \tera, a highly scalable derivative GP method based on target-specific exact gradient reduction. We prove that for stationary kernels, the gradient components orthogonal to the directions connecting the target and conditioning points are conditionally independent of the target function value; consequently, the exact conditional density is fully characterized by at most $m^2$ directional derivatives once a conditioning set of size $m$ is specified. By using these reduced, dimension-free conditionals as local factors in a Vecchia approximation, \tera effectively decouples $n$ and $d$ from the dense matrix inversion. This reduces the per-target evaluation cost to $\mathcal{O}(dm^2 + m^6)$ time and $\mathcal{O}(dm^2 + m^4)$ memory, leaving the underlying derivative GP model mathematically unchanged. Empirical evaluations demonstrate that \tera achieves state-of-the-art predictive accuracy while operating orders of magnitude faster than standard derivative GPs. Crucially, both computation time and peak GPU memory remain essentially flat with respect to $d$, enabling highly scalable inference in high-dimensional spaces. Code is available at \href{https://github.com/hseung88/tera}{https://github.com/hseung88/tera}.

\end{abstract}

\section{Introduction}
\label{sec:introduction}
Gradient observations are increasingly common in surrogate modeling, especially when data are generated by computational models equipped with automatic differentiation~\cite{Baydin2015AutomaticDI,Ament2022ScalableFB} or adjoint derivative solvers~\cite{Jameson1988AerodynamicDV}. The downstream target, however, is often not the gradient field itself. Gaussian-process (GP) regression requires calibrated prediction of a function response~\cite{Rasmussen2005GaussianPF}, and Bayesian optimization (BO) requires a surrogate that can locate extrema of a scalar objective~\cite{Wu2017BayesianOW}. Each query may therefore provide
a full $d$-dimensional gradient, while the surrogate is ultimately used through function-value predictions and decisions.

Derivative GP provide a principled model for such data by coupling function values and input derivatives through kernel derivatives~\cite{Solak2002DerivativeOI,Rasmussen2005GaussianPF}. Exact inference with $n$ function values and $n$ full gradients in $d$ dimensions involves a Gaussian state of size $n(d+1)$, which makes dense linear algebra and memory costs prohibitive when either $n$ or $d$ is moderate or large. Existing scalable derivative GP methods reduce this cost through kernel interpolation~\cite{Eriksson2018ScalingGP,huang2026scaling}, inducing derivative variables~\cite{Padidar2021ScalingGP}, and structured derivative covariance decompositions~\cite{Roos2021HighDimensionalGP}. These methods reduce the cost of inference in derivative GPs while keeping the joint function-gradient state as the main
computational object. We instead start from function-value prediction and 
use the target to determine how the observed gradients enter the relevant conditional distributions.

Gaussian conditioning turns this viewpoint into a precise reduction principle.
For a 
function-value target, 
after nearby function observations have been conditioned on, the observed gradients affect the target only through their contribution to its conditional mean and conditional variance.
Preserving this contribution is sufficient for the target prediction.
For a stationary prior kernel, the relevant contribution is not tied to the original $d$ gradient coordinates.
It is carried by linear 
functionals 
of observed gradients (i.e., directional derivatives) along directions spanned by
differences between the target input and the inputs where gradients are observed, with the kernel metric determining the geometry.
This reduces the derivative information needed for one scalar target, but it does not by itself control 
how many gradient locations enter the conditioning set.
\begin{figure}[tb]
    \centering
    \includegraphics[width=0.99\textwidth]{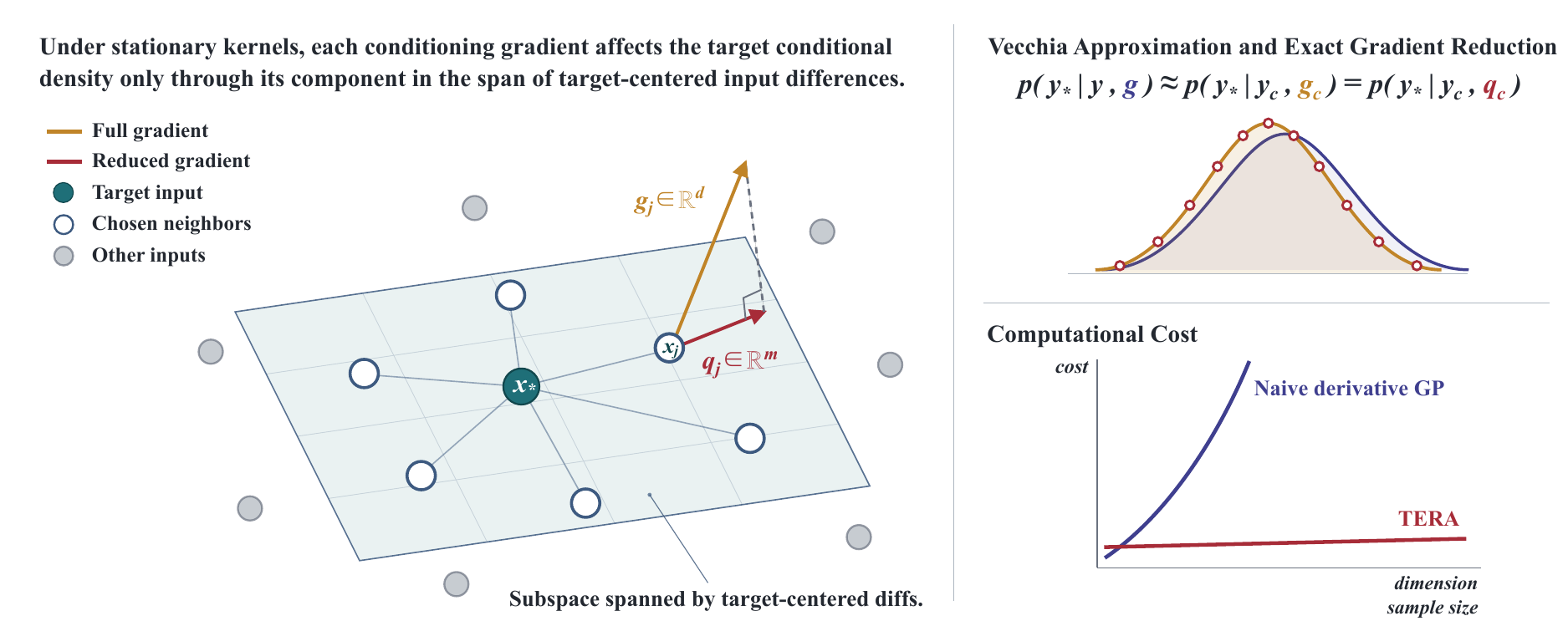}
    \caption{For predicting $y_*$ at input $\bx_*\in\bbR^{d}$ from observations $\by\in\bbR^n$ and gradients $\bg\in\bbR^{nd}$, \tera uses a Vecchia approximation to reduce the conditioning to a subset of $m$ nearest neighbors: $p(y_*|\by,\bg) \approx p(y_*|\by_c,\bg_c)$. Further, within each Vecchia conditional factor, \tera replaces the full local gradient $\bg_c\in\bbR^{md}$ by target-specific reduced gradients $\bq_c\in\bbR^{m^2}$ formed from its directional derivatives along the target-centered input differences. For stationary kernels, this preserves the full-gradient conditional density, leaving the Vecchia approximation as the only gap relative to exact derivative GP inference. \tera is highly scalable with respect to $d$ and $n$.}
    \label{fig:tera_overview}
\end{figure}

We control the number of conditioning locations by combining the reduction with Vecchia approximations~\cite{Vecchia1988EstimationAM,Stein2004,Datta2016,Katzfuss2020VecchiaAO,Katzfuss2021AGF,Schafer2020,Kang2021}. A target function value $f_\star$ is conditioned on $m$ neighboring locations rather than all observations, and the gradient reduction is applied within that conditioning set. The full gradients in the resulting Vecchia factor have dimension $md$, while the reduction for stationary kernels replaces them by at most $m^2$ 
target-specific gradient statistics.
This replacement 
gives the same conditional Gaussian distribution 
as the one obtained by conditioning on the full gradients in that set. 
These reduced conditional factors are used both for prediction and for evaluating the conditional Vecchia likelihood used to learn kernel and noise parameters.
With $m$ conditioning locations and $m\ll d$, each factor can be evaluated in $\order(dm^2+m^6)$ time, with only linear dependence on $d$ from reduced local block assembly.
For likelihood optimization, minibatching over Vecchia factors gives stochastic updates with cost $\order(|\mathcal B|(dm^2+m^6))$ for a minibatch $\mathcal B$, removing the explicit dependence on $n$ from each update when $|\mathcal B|$ is fixed. We refer to the resulting method as \tera, for Target-specific Exact gradient Reduction for Accelerating derivative GP inference. Figure~\ref{fig:tera_overview} illustrates the target-specific reduction inside a Vecchia factor and the resulting computational scaling of \tera.

The paper makes the following contributions.
\begin{itemize}[leftmargin=*,itemsep=1pt,topsep=2pt]
\item We prove an exact target-specific reduction for derivative GP conditionals. For stationary kernels, we show that a target function value conditioned on $m$ reference inputs depends on their $md$ gradient coordinates only through at most $m^2$ reduced gradients, under exact or metric-matched noisy gradient observations.

\item We turn this reduction into a scalable Vecchia derivative GP method. The reduced conditional covariance blocks are constructed directly from target-centered Gram matrices, avoiding the full $md \times md$ local derivative covariance. This yields per-factor cost $O(dm^2+m^6)$ and supports a conditional Vecchia learning objective over function values given observed gradients.

\item We show that the reduction preserves useful gradient information rather than merely lowering algebraic cost. Empirically, \tera matches the posterior accuracy of the corresponding full-gradient Vecchia derivative GP in fidelity tests, while turning gradients into better function-value prediction and lower BO regret at costs comparable to standard GP baselines.
\end{itemize}

\section{Related Work}
\label{sec:related}
\textbf{Scalable derivative GP inference. }
Recent scalable derivative GP methods reduce the cost of using derivative observations by changing the global covariance computation or posterior representation.
Kernel interpolation methods, including D-SKI~\cite{Eriksson2018ScalingGP} and DSoftKI~\cite{huang2026scaling}, accelerate derivative covariance computations through structured or learned interpolation, but their accuracy and fitting cost remain tied to a global interpolation approximation.
Inducing variable methods such as DSVGP and DDSVGP build sparse variational posterior approximations using derivative or directional derivative inducing variables~\cite{Padidar2021ScalingGP}, with directional derivatives reducing high-dimensional cost at the price of an additional directional approximation.
In contrast, we focus on settings where gradients are available but the inferential target is function-value prediction. In such settings, \tera expresses the gradient dependence of each chosen target conditional density through low-dimensional exact statistics, without introducing a global interpolation approximation or a variational inducing representation.

\textbf{Stationary kernel structures. }
The closest algebraic connection is de Roos et al.~\cite{Roos2021HighDimensionalGP}.
For stationary kernels, their decomposition expresses the gradient covariance matrix as the sum of a Kronecker product term and a low-rank correction, improving the scalability of exact 
derivative GP inference with respect to $d$ at the cost of  poor scalability with respect to $n$ ($\order(n^6)$ time).
In contrast, we use stationarity to identify the directions through which the conditioning gradients can affect a chosen target function value. When
applied inside Vecchia factors, the reduction makes the derivative solve depend on the conditioning set size $m$, rather than the total number of observations $n$, achieving scalability to both large $n$ and large $d$.

\textbf{Vecchia approximations for scalable GP inference. }
Vecchia approximations have been used in scalable GP inference~\cite{Cao2022ScalableGR,Cao2023VariationalSI} and BO~\cite{Jimenez2022ScalableBO}.
In existing methods, conditioning sets are chosen to approximate dependence among function-value observations.
A direct derivative GP analogue would still condition on all gradient coordinates from the chosen inputs.
We instead use each Vecchia factor with 
an exact lower-dimensional representation of the full gradient for the corresponding function-value prediction.

\section{Preliminaries}
\label{sec:prelim}
\subsection{Derivative GP Observation Model}
\label{sec:prelim_dgp}

We place a GP prior
$f(\cdot)\sim \GP(\mu_{\bftheta}(\cdot),k_{\bftheta}(\cdot,\cdot))$
on $\domain\subset\bbR^d$ and assume that $k_{\bftheta}$ is twice continuously differentiable in its input arguments.
For training inputs $\bX=(\bx_1,\ldots,\bx_n)^\top\in\bbR^{n\times d}$, write $f_i=f(\bx_i)$ and observe
$y_i=f_i+\epsilon_i^y$ and
$\bg_i=\nabla_{\bx}f(\bx_i)+\bfepsilon_i^g$.
The observation errors are independent across inputs and independent of the GP, 
with
$\epsilon_i^y\sim\normal(0,\sigma_y^2)$ and
$\bfepsilon_i^g\sim\normal(\bfzero,\sigma_g^2\bI_d)$.
We stack the observations as
$\by=(y_1,\ldots,y_n)^\top \in\bbR^n$ and
$\bg=(\bg_1^\top,\ldots,\bg_n^\top)^\top\in\bbR^{nd}$.

The pair $(\by,\bg)$ is jointly Gaussian,
\[
\begin{bmatrix}
\by \\
\bg
\end{bmatrix}
\sim
\normal\left(
\begin{bmatrix}
\bfmu_y \\
\bfmu_g
\end{bmatrix},
\begin{bmatrix}
\bK_{yy} & \bK_{yg} \\
\bK_{gy} & \bK_{gg}
\end{bmatrix}
\right),
\]
where $\bK_{yy}=\bK_{ff}+\sigma_y^2\bI_n$ and
$[\bK_{ff}]_{ij}=k_{\bftheta}(\bx_i,\bx_j)$.
The derivative covariance blocks are understood blockwise as
\[
[\bK_{yg}]_{i,j}
=
\nabla_{\bx_j} k_{\bftheta}(\bx_i,\bx_j)^\top,
\qquad
[\bK_{gg}]_{i,j}
=
\nabla_{\bx_i}\nabla_{\bx_j}^\top
k_{\bftheta}(\bx_i,\bx_j)
+
\mathbbm{1}\{i=j\}\sigma_g^2\bI_d.
\]
Exact gradient observations correspond to $\sigma_g^2=0$.

\subsection{Stationary Kernels and Target-Centered Differences}
\label{sec:prelim_stationary}

We use stationary kernels of the form
\begin{equation}
\label{eq:prelim_stationary_kernel}    
k_{\bftheta}(\bx,\bx')
=
\kappa_{\bftheta}\left((\bx-\bx')^\top\bfLambda_{\bftheta}(\bx-\bx')\right) = \kappa_{\bftheta}(\| \tilde\bx-\tilde\bx' \|^2),
\end{equation}
where $\bfLambda_{\bftheta}\succ0$ is diagonal and includes isotropic and ARD metrics~\cite{Rasmussen2005GaussianPF}, and $\tilde\bx=\bfLambda_{\bftheta}^{1/2}\bx$ denotes the scaled input.

For a target input $\bx_i$ and 
reference inputs $\bx_{j_1},\ldots,\bx_{j_m}$, define
\[
\bD_i
=
\left[
\bx_{j_1}-\bx_i,
\ldots,
\bx_{j_m}-\bx_i
\right]
\in\bbR^{d\times m},
\qquad
\bfDelta_i
=
\bfLambda_{\bftheta}^{1/2}\bD_i
=
\left[
\tilde\bx_{j_1}-\tilde\bx_i,
\ldots,
\tilde\bx_{j_m}-\tilde\bx_i
\right].
\]
The corresponding Gram matrix is
\[
\bH_i
=
\bfDelta_i^\top\bfDelta_i
=
\bD_i^\top\bfLambda_{\bftheta}\bD_i,
\qquad
[\bH_i]_{ab}
=
(\bx_{j_a}-\bx_i)^\top
\bfLambda_{\bftheta}
(\bx_{j_b}-\bx_i).
\]
Thus $\bfDelta_i$ contains the target-centered differences in scaled input coordinates.
The matrix $\bH_i$ contains their inner products, from which all pairwise squared distances among the target and reference inputs can be obtained.
For example,
$\|\tilde\bx_{j_a}-\tilde\bx_{j_b}\|^2
=
[\bH_i]_{aa}+[\bH_i]_{bb}-2[\bH_i]_{ab}$.

\subsection{Vecchia Ordering and Conditioning Sets}
\label{sec:prelim_vecchia}

Vecchia approximations~\cite{Vecchia1988EstimationAM,Katzfuss2021AGF} replace full chain-rule conditioning in an ordered Gaussian model by small conditioning sets.
For an ordered vector $\by=(y_1,\ldots,y_n)^\top$,
\[
p(\by)
=
\prod_{i=1}^n p(y_i \mid y_1,\ldots,y_{i-1})
\approx
\prod_{i=1}^n p\bigl(y_i \mid \by_{c(i)}\bigr),
\]
where $c(i)\subset\{1,\ldots,i-1\}$ and $|c(i)|\le m$.
We use a maximum-minimum-distance (maximin) ordering and take $c(i)$ to be the $m_i$ nearest previously ordered neighbors of $\bx_i$ in the $\bfLambda_{\bftheta}^{1/2}$-scaled input space, where $m_i=|c(i)|$, following scaled Vecchia constructions~\cite{Katzfuss2020ScaledVA}.

\section{Methods}
\label{sec:method}
\subsection{Targeted Gradient Reduction for Function Values}
\label{sec:method_targeted_reduction}

We first state the reduction for a generic finite reference set.
The result does not depend on how the reference set is chosen and will later be applied to Vecchia conditioning sets for prediction and parameter estimation.

Fix a target input $\bx_i\in\bbR^d$ and a 
reference set
$c=\{j_1,\ldots,j_{m_c}\}$ with $i\notin c$.
Let $\by_c=(y_{j_1},\ldots,y_{j_{m_c}})^\top$ and
$\bg_c=(\bg_{j_1}^\top,\ldots,\bg_{j_{m_c}}^\top)^\top\in\bbR^{m_c d}$.
Let $\bD_i\in\bbR^{d\times m_c}$ be the target-centered difference matrix from Section~\ref{sec:prelim_stationary} for the 
reference inputs in $c$.
Define the target-specific reduced gradient statistic $\bq_c$ as
\begin{equation}
\label{eq:reduced_gradient_statistic}
\bq_c
=
(\bI_{m_c}\otimes \bD_i^\top)\bg_c
=
\begin{bmatrix}
\bD_i^\top \bg_{j_1}\\
\vdots\\
\bD_i^\top \bg_{j_{m_c}}
\end{bmatrix}\in\bbR^{m_c^2},
\qquad
\bD_i^\top \bg_{j_a}
=
\begin{bmatrix}
(\bx_{j_1}-\bx_i)^\top \bg_{j_a}\\
\vdots\\
(\bx_{j_{m_c}}-\bx_i)^\top \bg_{j_a}
\end{bmatrix}
\in\bbR^{m_c}.
\end{equation}
The entries of $\bq_c$ can be viewed as unnormalized directional derivatives.
Each term $(\bx_{j_b}-\bx_i)^\top\bg_{j_a}$ is the inner product between the observed
gradient at $\bx_{j_a}$ and the target-centered difference $\bx_{j_b}-\bx_i$.
These directions are determined by the chosen target input and reference set, rather than
selected as global derivative directions.
When the target input and reference set are clear, we refer to $\bq_c$ simply as the reduced gradient.
The reduction is target-specific because $\bD_i$ depends on the target input $\bx_i$.

The basic object used by \tera is the conditional density
$p_{\bftheta}(y_i\mid \by_c,\bq_c)$.
The next lemma gives its Schur-complement form. All proofs can be found in Appendix~\ref{app:proofs_main}.

\begin{lemma}[Conditional density with reduced gradients]
\label{lem:reduced_gradient_factor}
Under the joint Gaussian model for $(y_i,\by_c,\bq_c)$,
\begin{equation}
\label{eq:reduced_gradient_factor}
p_{\bftheta}(y_i\mid\by_c,\bq_c)
=
\normal\left(
\mu_{y_i\mid \by_c}+t_i,\,
\sigma^2_{y_i\mid \by_c}-s_i
\right),
\end{equation}
where
\[
t_i
=
\bk_{y_i\bq_c\mid \by_c}
\bK_{\bq_c\bq_c\mid \by_c}^{-1}
\left(
\bq_c-\bfmu_{\bq_c\mid \by_c}
\right),
\qquad
s_i
=
\bk_{y_i\bq_c\mid \by_c}
\bK_{\bq_c\bq_c\mid \by_c}^{-1}
\bk_{\bq_c y_i\mid \by_c}.
\]
Here $\mu_{y_i\mid \by_c}$, $\sigma^2_{y_i\mid \by_c}$,
$\bfmu_{\bq_c\mid \by_c}$,
$\bk_{y_i\bq_c\mid \by_c}$, and
$\bK_{\bq_c\bq_c\mid \by_c}$ are Gaussian conditional moments after conditioning on $\by_c$.
\end{lemma}

The terms $t_i$ and $s_i$ are the mean correction and variance reduction induced by conditioning on the reduced gradient.
For stationary kernels, we will show that this reduced gradient preserves the same conditional density obtained by conditioning on the full reference gradients. This lets us evaluate the same target conditional density using reduced covariance blocks (whose size does not depend on $d$) rather than forming and factorizing full derivative covariance blocks.

\begin{proposition}[Targeted gradient reduction]
\label{prop:targeted_gradient_reduction}
Assume the stationary kernel model in \eqref{eq:prelim_stationary_kernel}. 
Suppose that the observed gradients are either exact, so that $\sigma_g^2=0$, or have independent Gaussian noise matched to the scaled input metric, $\bfepsilon_i^g\sim\normal(\bfzero,\sigma_g^2\bfLambda_{\bftheta})$.
For any target input $\bx_i$ and any reference set
$c=\{j_1,\ldots,j_{m_c}\}$ with $i\notin c$,
\begin{equation}
\label{eq:reduced_gradient_exactness}
p_{\bftheta}(y_i\mid \by_c,\bg_c)
=
p_{\bftheta}(y_i\mid \by_c,\bq_c).
\end{equation}
Thus the full reference gradient vector in $\bbR^{m_c d}$ can be replaced by the reduced gradient $\bq_c\in\bbR^{m_c^2}$ without changing the conditional density for the chosen function-value target.
\end{proposition}

The proof is given in Appendix~\ref{app:proof_prop1}, and it uses the target-centered differences defined in Section~\ref{sec:prelim_stationary}.
The matched noise condition is automatically satisfied by exact gradients and by isotropic kernels under the iid unscaled gradient noise model in Section~\ref{sec:prelim_dgp}.
For other noise models, conditioning on the reduced gradient still defines a valid Gaussian conditional density, but exact equivalence to full noisy gradient conditioning need not hold.
We keep the stacked form in \eqref{eq:reduced_gradient_statistic} because it avoids basis construction and leads directly to the reduced block formulas. 

\paragraph{Direct Construction and Scalability.}
The primary advantage of replacing the full reference gradient $\bg_c$ with the reduced gradient $\bq_c$ is computational. Evaluating the conditional density $p_{\bftheta}(y_i\mid\by_c,\bq_c)$ requires 
the conditional blocks in Lemma~\ref{lem:reduced_gradient_factor}.
A naive approach would form the full 
derivative covariance of $\bg_c$, 
an $m_c d\times m_c d$ matrix, and then apply the linear map in \eqref{eq:reduced_gradient_statistic}.
This route retains the full-gradient state and incurs a prohibitive $\order(m_c^3 d^3)$ dense factorization cost in large $d$.
Instead, we construct the reduced blocks directly from the stationary kernel geometry.

\begin{lemma}[Direct construction of reduced blocks]
\label{lem:direct_blocks}
For a target input $\bx_i$ and reference set $c$, the 
reduced 
blocks 
$\bfmu_{\bq_c\mid \by_c}$,
$\bk_{y_i\bq_c\mid \by_c}$, and
$\bK_{\bq_c\bq_c\mid \by_c}$
needed in Lemma~\ref{lem:reduced_gradient_factor} can be computed directly from the elements of the $m_c \times m_c$ target-centered Gram matrix $\bH_i=\bD_i^\top\bfLambda_{\bftheta}\bD_i$. This bypasses the instantiation of the full $m_c d \times m_c d$ derivative covariance, yielding the $m_c^2 \times m_c^2$ reduced blocks directly.
\end{lemma}

The explicit algebraic formulas for these blocks are provided in Appendix~\ref{app:projected_blocks}. Crucially, the dense Cholesky factorization of the resulting reduced conditional covariance $\bK_{\bq_c\bq_c\mid \by_c}$ costs $\order(m_c^6)$. This eliminates the cubic dependence on the input dimension $d$ from the full 
derivative computation, allowing our method to scale to high-dimensional problems where standard derivative GPs can fail.

\subsection{Conditional Vecchia Likelihood with Reduced Gradients}
\label{sec:method_vecchia_learning}

We now apply the target-specific reduction from Section~\ref{sec:method_targeted_reduction} to the Vecchia conditioning sets from Section~\ref{sec:prelim_vecchia}.
For the $i$th training factor, write
$c(i)=\{j_1,\ldots,j_{m_i}\}$ with $m_i=|c(i)|$.
Let $\bD_i\in\bbR^{d\times m_i}$ be the target-centered difference matrix formed from $\bx_i$ and the ordered inputs in $c(i)$.
The reduced gradient vector for the $i$th factor is $
\bq_{c(i)}
=
(\bI_{m_i}\otimes \bD_i^\top)\bg_{c(i)}
\in\bbR^{m_i^2}$.
This is the specialization of $\bq_c$ in Section~\ref{sec:method_targeted_reduction} to the Vecchia set $c(i)$.
Define the conditional Vecchia density with reduced gradients as
\begin{equation}
\label{eq:conditional_vecchia_density}
\tilde p_{\bftheta}^{\mathrm{Vec}}(\by\mid\bg)
=
\prod_{i=1}^n
p_{\bftheta}\left(y_i\mid \by_{c(i)},\bq_{c(i)}\right).
\end{equation}
Each factor is a normalized Gaussian density in $y_i$.
Since $c(i)\subset\{1,\ldots,i-1\}$, sequential integration over $y_n,\ldots,y_1$ shows that \eqref{eq:conditional_vecchia_density} is a normalized conditional density over $\by$ for fixed observed gradients.

We estimate hyperparameters by maximizing
\begin{equation}
\label{eq:conditional_vecchia_loglik}
\tilde\ell_n(\bftheta)
= \log \tilde p_{\bftheta}^{\mathrm{Vec}}(\by\mid\bg) = 
\sum_{i=1}^n
\log p_{\bftheta}\left(y_i\mid \by_{c(i)},\bq_{c(i)}\right).
\end{equation}
The parameter vector $\bftheta$ includes the kernel parameters and the noise levels $\sigma_y^2$ and $\sigma_g^2$.
For fixed conditioning sets, the reduced gradients $\bq_{c(i)}$ are deterministic functions of the observed gradients and inputs.
The Gaussian covariance blocks used to evaluate each factor depend on $\bftheta$.
By applying Lemma~\ref{lem:direct_blocks} to the Vecchia conditioning sets with $m_i\le m$, the inverse computation in each factor costs at most $\order(m^6)$ via Cholesky factorization; importantly, this cost is completely independent of $d$.

\begin{proposition}[Score unbiasedness]
\label{prop:vecchia_score}
Assume the data are generated from the derivative GP model at parameter $\bftheta_0$.
Suppose the ordering and conditioning sets are fixed, and differentiation may be interchanged with integration.
Then the reduced conditional Vecchia score satisfies
\[
\mathbb E_{\bftheta_0}
\left[
\nabla_{\bftheta}\tilde\ell_n(\bftheta_0)
\right]
=
\bfzero.
\]
\end{proposition}

Proposition~\ref{prop:vecchia_score} shows that the score of
$\tilde\ell_n(\bftheta)$ is unbiased at the data generating parameter
$\bftheta_0$ for fixed ordering and conditioning sets, resulting in unbiased estimating equations even though the Vecchia product is not the full joint derivative GP likelihood.
In particular, optimization of \eqref{eq:conditional_vecchia_loglik}
targets the conditional distribution of function values given observed
gradients, without evaluating the marginal density of $\bg$. This aligns the learning objective with the inferential target of the paper, where gradients are used as conditioning information and prediction and downstream decisions depend on function values.

During each inner optimization stage, the ordering and conditioning sets remain fixed.
Since \eqref{eq:conditional_vecchia_loglik} is a sum over conditional factors, a minibatch $\mathcal B\subset\{1,\ldots,n\}$ gives the unbiased score estimator
\[
\widehat{\nabla}\tilde\ell_B(\bftheta)
=
\frac{n}{|\mathcal B|}
\sum_{i\in \mathcal B}
\nabla_{\bftheta}
\log p_{\bftheta}\left(y_i\mid \by_{c(i)},\bq_{c(i)}\right),
\]
following factor subsampling strategies for Vecchia GP likelihoods~\cite{Cao2022ScalableGR,Jimenez2022ScalableBO}.
When the scaled metric used for neighbor selection is updated, we optionally recompute the maximin ordering and nearest-neighbor conditioning sets and continue with a new fixed-set inner optimization stage, although this recomputing can be skipped for exponentially increasing numbers of iterations during optimization~\cite{Katzfuss2020ScaledVA}.

\subsection{Prediction and Computational Complexity}
\label{sec:prediction_complexity}

\textbf{Prediction. }
Prediction uses the same local reduced factor as training.
For a new input $\bx_\star$, choose a Vecchia conditioning set $c(\star)$ from the observed inputs (usually nearest neighbors) using the same scaled input geometry as in Section~\ref{sec:prelim_vecchia}.
Write $m_\star=|c(\star)|$ and form the target-centered difference matrix $\bD_\star\in\bbR^{d\times m_\star}$ from $\bx_\star$ and the ordered inputs in $c(\star)$.
The reduced gradient vector for the prediction factor is
\[
\bq_{c(\star)}
=
(\bI_{m_\star}\otimes\bD_\star^\top)\bg_{c(\star)}
\in\bbR^{m_\star^2}.
\]
We then evaluate
\[
p_{\bftheta}\left(f_\star\mid \by_{c(\star)},\bq_{c(\star)}\right).
\]
with the reduced conditional blocks constructed as in Appendix~\ref{app:projected_blocks}. The predictive variance is the latent variance of $f_\star$. If predicting a noisy future observation, $\sigma_y^2$ is added to this variance. When multiple prediction inputs use the same conditioning set, the Cholesky
factor of $\bK_{\by_c\by_c}$ can be reused. The reduced gradient and the reduced covariance blocks that depend on $\bD_\star$ remain target-specific and are recomputed for each prediction factor.

\textbf{Computational complexity. }
For a factor with at most $m$ conditioning inputs, direct construction of the reduced
blocks costs $\order(dm^2)$ time, while the dense Cholesky factorization of the reduced
conditional covariance costs $\order(m^6)$. 
Since the objective in \eqref{eq:conditional_vecchia_loglik} is a sum over Vecchia factors,
a stochastic update using a minibatch $\mathcal B$ costs
\[
\order\left(|\mathcal B|(dm^2+m^6)\right) \quad\text{time},\qquad
\order(m^2d + m^4)
\quad\text{memory}
\]
under sequential evaluation. For fixed $m$ and fixed minibatch size, each stochastic likelihood update is linear
in the input dimension $d$ and independent of $n$ once the ordering and conditioning sets are fixed. A full likelihood evaluation, or one complete pass over
all training factors, costs $\order\left(n(dm^2+m^6)\right)$.

\section{Experiments}
\label{sec:experiment}
In this section, we evaluate whether \tera preserves the predictive contribution of observed gradients while reducing the cost of derivative conditioning. We first use GP simulations, where exact derivative GP posterior marginals are available, to separate posterior fidelity from
model misspecification. We then evaluate large-scale GP regression, where exact derivative GP inference is no longer practical, and high-dimensional BO, where gradients inform the surrogate
but decisions depend only on function values. The experiments are implemented in a PyTorch~\cite{Paszke2019PyTorchAI}
code base using GPyTorch~\cite{Gardner2018GPyTorchBM} and
BoTorch~\cite{Balandat2019BoTorchPB} where applicable. Wall-clock time and peak GPU memory are measured on
a single NVIDIA H100 GPU. Experimental details are provided in
Appendix~\ref{app:experiment_detail}, and ablation studies are given in
Appendix~\ref{app:ablation}.

\subsection{Posterior Fidelity and Scaling}
\label{sec:exp_gp_prior_simulation}

\begin{figure}
    \centering
    \includegraphics[width=0.99\textwidth]{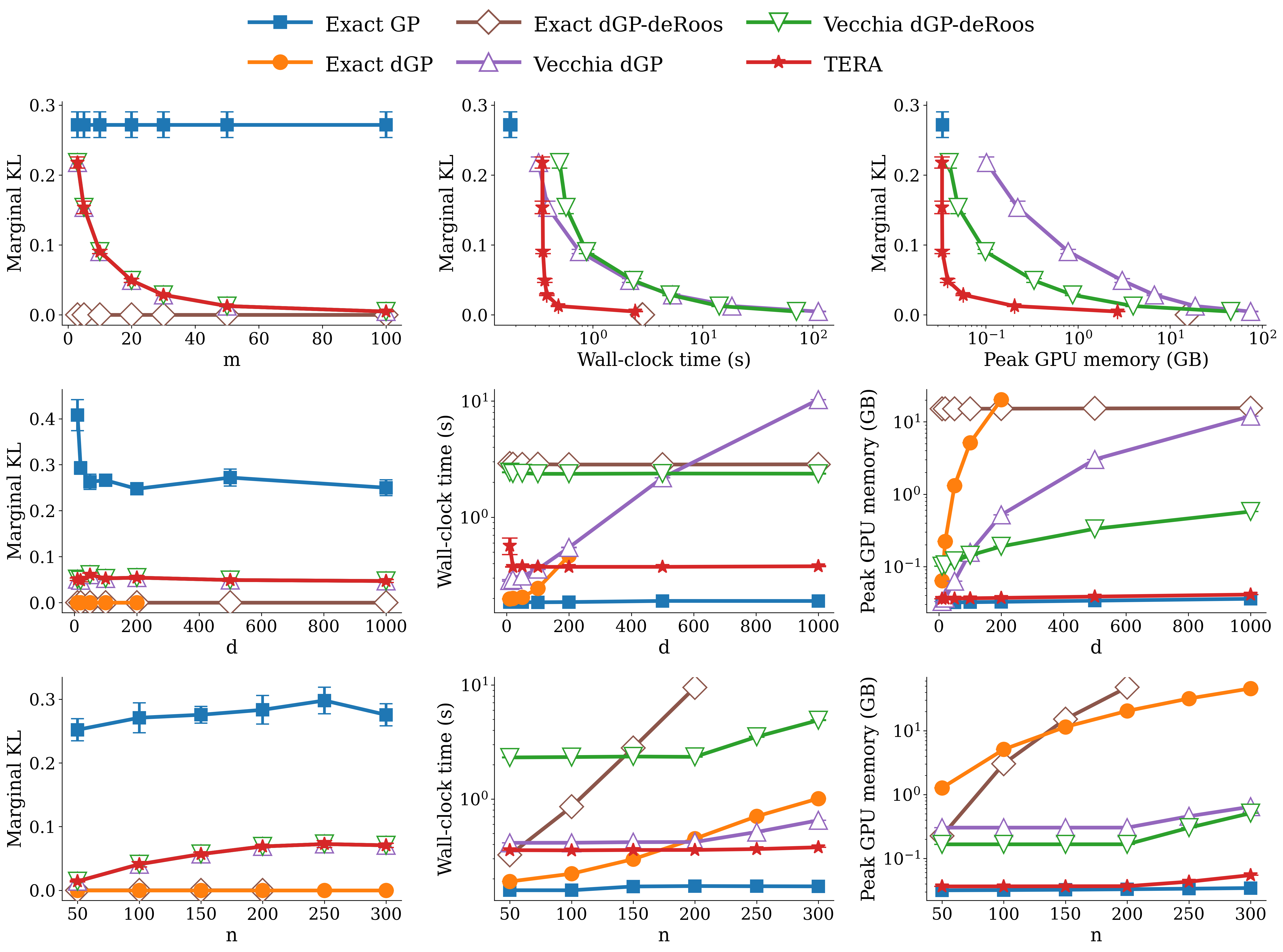}
    \caption{
    Posterior fidelity and computational scaling under GP simulations. Rows show sweeps over the conditioning size $m$ at $(d,n)=(500,150)$, input dimension $d$ at $(m,n)=(20,150)$, and training size $n$ at $d=150$ with $m=\max(20,n/10)$. Columns report average marginal KL divergence to the exact derivative GP posterior, wall-clock time, and peak GPU memory. All simulations use 100 common evaluation points and five different runs.
    }
    \label{fig:posterior_fidelity_scaling}
\end{figure}

We first evaluate posterior fidelity in a controlled GP-simulated setting.
Training and evaluation inputs are generated independently, and noiseless function values and gradients are sampled from an exact derivative GP with a Mat\'ern-5/2 kernel.
We fix the output scale to 1 and choose the isotropic lengthscale so that two training inputs at the median pairwise distance have prior correlation 0.3. This keeps the typical prior correlation comparable across dimensions, so that changes in posterior fidelity are not confounded by changes in effective prior smoothness as $d$ varies. Appendix~\ref{app:target_corr_ablation} shows that the posterior-fidelity pattern is robust to the choice of target prior correlation.
The exact derivative GP posterior is computed by dense algebra when feasible and by the de Roos decomposition~\cite{Roos2021HighDimensionalGP} otherwise.
We measure posterior fidelity by the average of marginal KL divergence $\KL(p_j^\star \,\|\, p_j^{\mathcal M})$, where $p_j^\star$ is the exact derivative GP marginal and $p_j^{\mathcal M}$ is the marginal from method $\mathcal M$.

We compare \tera with function-only Exact GP, dense Exact derivative GP (Exact dGP), Exact dGP-deRoos, full-gradient Vecchia dGP, and Vecchia dGP-deRoos.
Exact dGP-deRoos computes the same posterior as dense Exact dGP using the de Roos decomposition, while Vecchia dGP-deRoos is the full-gradient Vecchia counterpart accelerated by the same algebra.
The Vecchia baselines use the same ordering and conditioning sets as \tera but condition on full local gradients.

\textbf{\tera preserves the gradient information that closes the gap to Exact dGP. }
The left column of Figure~\ref{fig:posterior_fidelity_scaling} shows posterior approximation error relative to the exact derivative GP posterior.
Exact GP remains far from this reference, indicating that the observed gradients contribute substantial information to function-value prediction.
full-gradient Vecchia dGPs narrow the gap, with fidelity improving as $m$ increases.
Across all sweeps, \tera overlaps with the full-gradient Vecchia baselines, consistent with Proposition~\ref{prop:targeted_gradient_reduction}.
Thus, the remaining KL gap to Exact dGP is due to Vecchia approximations, not to replacing full local gradients by reduced gradients.

\textbf{\tera achieves Vecchia dGP fidelity at substantially lower cost. }
The middle and right columns show wall-clock time and peak GPU memory.
\tera matches the posterior fidelity of full-gradient Vecchia dGPs while running much faster in the $m$ sweep and remaining close to Exact GP timing across the $d$ and $n$ sweeps.
Its peak memory is comparable to Exact GP for $m<50$, while Vecchia dGP explicitly carries all $md$ local gradient coordinates.
Exact dGP becomes infeasible as $d$ or $n$ grows, and Exact dGP-deRoos shows rapidly increasing memory cost with $n$.

\subsection{Large-Scale GP Regression with Observed Gradients}

We evaluate large-scale GP regression on the MD22 molecular benchmark~\cite{Chmiela2022AccurateGM}, following the experimental setup of \cite{huang2026scaling}.
Inputs are molecular configurations, responses are scalar energies, and observed forces provide gradient information.
We use six benchmarks, DHA, AT-AT, Stachyose, AT-AT-CG-CG, Buckyball-catcher, and Double-walled-nanotube, spanning 4,528 to 62,777 samples and 168 to 1,110 input dimensions.
The prediction target is held-out scalar energy, while gradients are used only as conditioning information.

We compare \tera with Exact GP, DDSVGP~\cite{Padidar2021ScalingGP}, and DSoftKI~\cite{huang2026scaling} on identical train and test splits over three different split seeds.
For DDSVGP and DSoftKI, energies and forces are standardized using the joint energy-force scale from \cite{huang2026scaling} while \tera standardizes the energy target using the energy scale.
We therefore report raw energy RMSE per atom after converting predictions back to physical units. Since the DDSVGP implementation does not support ARD or Mat\'ern kernels, the main comparison uses squared exponential (SE) kernels without ARD for all methods. Additional results, including Mat\'ern-5/2 kernels and ARD, are provided in Appendix~\ref{app:additional_experiment}.

\begin{figure}[tb]
    \centering
    \includegraphics[width=0.99\textwidth]{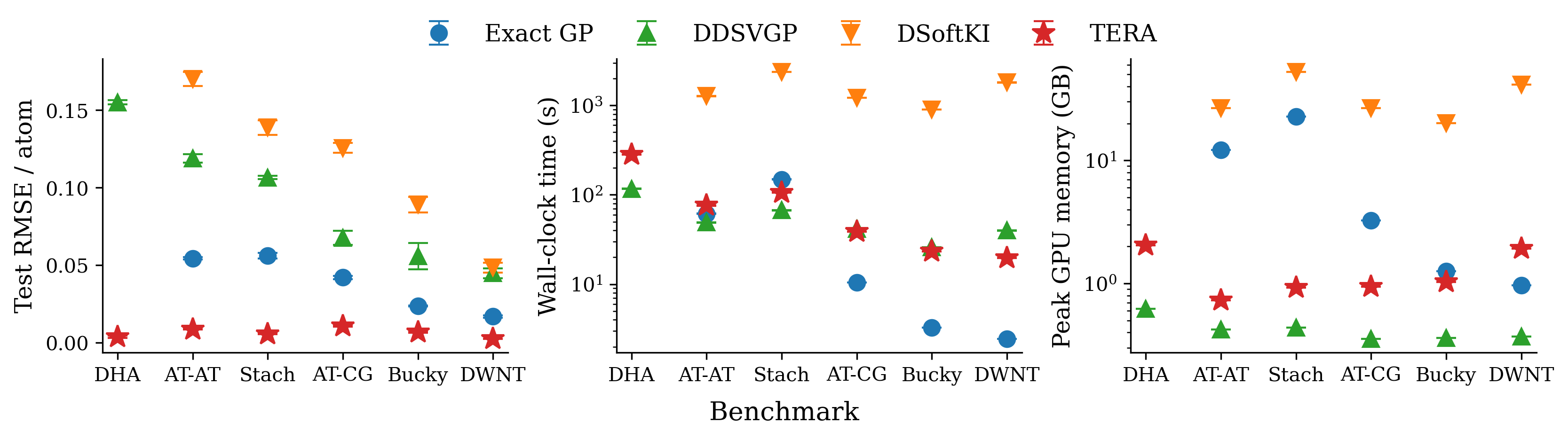}
    \caption{Large-scale scalar energy prediction on MD22. Panels show raw energy test RMSE per atom after conversion to physical units, end-to-end wall-clock time, and peak GPU memory. All methods use isotropic squared exponential kernels and identical train and test splits. Results are averaged over three independent runs. Lower values are better in all panels.}
    \label{fig:md22_main_metrics_rbf}
\end{figure}

\textbf{\tera achieves the strongest prediction accuracy with a lightweight learning problem.}
Figure~\ref{fig:md22_main_metrics_rbf} shows that \tera gives the lowest raw energy RMSE per atom across the MD22 systems.
DDSVGP and DSoftKI use gradients through larger approximate full-derivative GP objectives, whereas \tera learns only the kernel and noise parameters in the conditional Vecchia objective.
Empirically, one pass over the training factors is sufficient for \tera to obtain the reported accuracy, while DDSVGP and DSoftKI need to be trained for multiple epochs. See Figure~\ref{fig:md22_training_curves_rbf} in Appendix~\ref{app:additional_experiment}.

\textbf{\tera remains computationally practical.}
The timing and memory panels show that the accuracy gain from observed gradients remains practical in both large sample and high dimensional regimes.
On DHA, the largest system with 62,777 samples and 168 dimensions, \tera scales to the full training set while using gradients, whereas Exact GP and DSoftKI do not complete under the same resource budget.
On Double-walled-nanotube, the highest dimensional system with 4,528 samples and 1,110 dimensions, \tera remains memory-efficient despite the large gradient dimension.

\subsection{High-Dimensional Bayesian Optimization}
\label{sec:bo_experiments}

We provide BO experiments as a downstream test of whether \tera's surrogate-level gradient reduction translates into better high-dimensional scalar objective optimization under a given query budget.
We consider a setting in which each query returns a scalar objective value and its full gradient.
We use synthetic benchmarks whose objectives depend on all input dimensions, rather than benchmarks with a small embedded active subspace.
We use two full-dimensional synthetic benchmarks, Ackley-200D and Levy-200D, covering rugged and multimodal objectives. Results are averaged over five independent runs with different initial designs.
We report simple regret against the number of function queries, together with end-to-end BO wall-clock time and peak GPU memory.

We compare with Sobol, dimension-scaled vanilla BO (VBO)~\cite{Hvarfner2024VanillaBO}, and TuRBO-1~\cite{Eriksson2019ScalableGO}.
Sobol provides a nonadaptive space-filling baseline. VBO is a modern global GP baseline for high-dimensional BO that uses lengthscale priors scaled with the input dimension, while TuRBO-1 is a strong local trust-region baseline.
We use TuRBO-1 rather than multi-region variants to keep a sequential evaluation with one new query per BO iteration, without introducing allocation across multiple trust regions. All adaptive BO methods use LogEI~\cite{Ament2023UnexpectedIT}, with TuRBO-1 optimizing it inside the current trust region, so the acquisition rule is aligned across GP surrogates.
An additional Ackley-200D comparison using the original large batch TuRBO setup is provided in Appendix~\ref{app:additional_experiment}.

\begin{figure}[tb]
    \centering
    \includegraphics[width=0.98\textwidth]{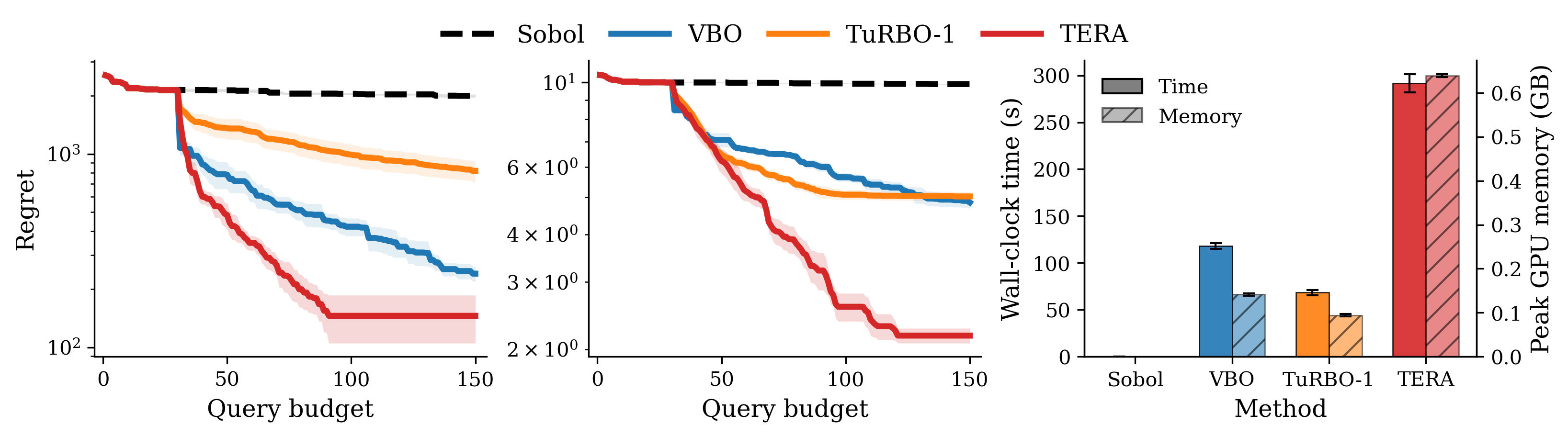}
    \caption{High-dimensional BO under a fixed budget of 150 objective evaluations. The first two panels report simple regret on full-dimensional Levy-200D and Ackley-200D, respectively. Curves show means and standard errors over five runs. The right panel reports wall-clock time and peak GPU memory on Ackley-200D, with error bars showing standard errors.}
    \label{fig:bo_main_regret}
\end{figure}

\textbf{\tera lowers regret at practical BO cost.}
Figure~\ref{fig:bo_main_regret} shows that \tera obtains the lowest final regret on both high-dimensional benchmarks under the same query budget, converting available gradient information into substantially better BO decisions.
On Levy-200D, the gap is large. VBO and TuRBO-1 improve over Sobol after the initial design, but their progress slows at substantially higher regret, while \tera continues to make larger improvements before stabilizing at a lower level.
On Ackley-200D, the same ordering appears even more clearly.
The value-only BO baselines reduce regret, but they plateau well above \tera, indicating that their scalar observations alone do not provide enough local information to guide the search as effectively in this 200-dimensional space.

The resource panel shows that \tera uses more computation than value-only BO, but the additional cost is moderate relative to the improvement in final regret. For example, the full Ackley-200D run still takes only about five minutes (i.e., roughly 2 seconds on average per query update) and remains below one GB of peak GPU memory.

\subsection{Discussion and Limitations}
\label{sec:limitations}

\tera reduces the dependence on the input dimension, but the reduced local solve still depends on the Vecchia conditioning size.
With $m$ conditioning inputs, the reduced conditional covariance has size at most $m^2\times m^2$, so each Vecchia factor includes an $\order(m^6)$ dense Cholesky term.
The method can therefore be computationally costly when accurate prediction requires large conditioning sets.
Our experiments show robust performance with small conditioning sets, around $m=20$, across the tested GP simulations, GP regression tasks, and BO downstream tasks, but this should not be read as a universal guarantee.
Additional ablations on the effect of $m$ are discussed in Appendix~\ref{app:effect_m}.

For general anisotropic or correlated gradient noise, conditioning on the reduced gradients still defines a valid Gaussian conditional factor, but it need not match conditioning on the full noisy gradient vector.
Our experiments focus on noise settings covered by Proposition~\ref{prop:targeted_gradient_reduction}. The main GP regression experiments use iid unscaled gradient noise with an isotropic SE kernel, and the ARD Mat\'ern-5/2 experiments in Appendix~\ref{app:md22_additional_experiment} use metric-matched gradient noise.
Exact reductions for arbitrary gradient noise remain outside the scope of this work.

\section{Conclusions}
\label{sec:conclusion}
We introduced \tera, a scalable derivative GP method that uses observed gradients as conditioning information for function-value prediction, rather than treating the full derivative state as the main computational object. For stationary kernels, each local function-value conditional depends on the full gradients only through target-specific reduced gradients, which can be embedded in Vecchia factors without changing the corresponding local conditional distribution. 
Empirically, the reduction matches the posterior fidelity of full-gradient Vecchia derivative GP inference and turns observed gradients into improved function prediction and lower high-dimensional BO regret at practical computational cost.
The results suggest that scalable derivative GP inference should be organized around the predictive target, not around all available derivative coordinates.

\bibliographystyle{plain}
\bibliography{main}


\appendix
\clearpage
\appendix
\section*{Appendix}
\addcontentsline{toc}{section}{Appendix}

\section{Proofs for Section~\ref{sec:method}}
\label{app:proofs_main}

\subsection{Proof of Lemma~\ref{lem:reduced_gradient_factor}}

\begin{proof}
Condition on $\by_c$.
Under the joint Gaussian model, $(y_i,\bq_c)$ remains jointly Gaussian with conditional mean
\[
\begin{bmatrix}
\mu_{y_i\mid \by_c}\\
\bfmu_{\bq_c\mid \by_c}
\end{bmatrix}
\]
and conditional covariance
\[
\begin{bmatrix}
\sigma^2_{y_i\mid \by_c}
&
\bk_{y_i\bq_c\mid \by_c}
\\
\bk_{\bq_c y_i\mid \by_c}
&
\bK_{\bq_c\bq_c\mid \by_c}
\end{bmatrix}.
\]
Applying the Gaussian conditioning formula to $y_i$ given $\bq_c$ under this conditional Gaussian distribution gives
\[
p_{\bftheta}(y_i\mid \by_c,\bq_c)
=
\normal\left(
\mu_{y_i\mid \by_c}
+
\bk_{y_i\bq_c\mid \by_c}
\bK_{\bq_c\bq_c\mid \by_c}^{-1}
(\bq_c-\bfmu_{\bq_c\mid \by_c}),
\,
\sigma^2_{y_i\mid \by_c}
-
\bk_{y_i\bq_c\mid \by_c}
\bK_{\bq_c\bq_c\mid \by_c}^{-1}
\bk_{\bq_c y_i\mid \by_c}
\right).
\]
This is exactly \eqref{eq:reduced_gradient_factor}.
\end{proof}

\subsection{Proof of Proposition~\ref{prop:targeted_gradient_reduction}}
\label{app:proof_prop1}

\begin{proof}
Fix a target input $\bx_i$ and an ordered reference set
$c=\{j_1,\ldots,j_{m_c}\}$.
Work in scaled input coordinates
\[
\tilde\bx=\bfLambda_{\bftheta}^{1/2}\bx.
\]
For gradients, write
\[
\tilde\nabla f(\tilde\bx)
=
\bfLambda_{\bftheta}^{-1/2}\nabla_{\bx}f(\bx),
\qquad
\tilde\bg_j
=
\bfLambda_{\bftheta}^{-1/2}\bg_j.
\]
Under exact gradient observations, $\tilde\bg_j=\tilde\nabla f(\tilde\bx_j)$.
Let
\[
\mathcal S_i
=
\operatorname{span}\{\tilde\bx_{j_a}-\tilde\bx_i:a=1,\ldots,m_c\}
=
\operatorname{col}(\bfDelta_i).
\]
Every difference between two inputs in $\{i\}\cup c$ lies in $\mathcal S_i$, since
\[
\tilde\bx_{j_a}-\tilde\bx_{j_b}
=
(\tilde\bx_{j_a}-\tilde\bx_i)
-
(\tilde\bx_{j_b}-\tilde\bx_i).
\]
Let $\mathcal S_i^\perp$ denote the orthogonal complement of $\mathcal S_i$ in scaled coordinates.
For the stationary kernel
\[
k_{\bftheta}(\bx_a,\bx_b)
=
\kappa_{\bftheta}\left(\|\tilde\bx_a-\tilde\bx_b\|^2\right),
\]
write $\br_{ab}=\tilde\bx_a-\tilde\bx_b$ and $r_{ab}=\|\br_{ab}\|^2$.
The required derivative identities are
\[
\nabla_{\tilde\bx_a} k_{\bftheta}(\bx_a,\bx_b)
=
2\kappa_{\bftheta}'(r_{ab})\br_{ab},
\qquad
\nabla_{\tilde\bx_b} k_{\bftheta}(\bx_a,\bx_b)
=
-2\kappa_{\bftheta}'(r_{ab})\br_{ab},
\]
and
\begin{equation}
\label{eq:mixed_derivative_identity}
\nabla_{\tilde\bx_a}\nabla_{\tilde\bx_b}^{\top} k_{\bftheta}(\bx_a,\bx_b)
=
-2\kappa_{\bftheta}'(r_{ab})\bI_d
-
4\kappa_{\bftheta}''(r_{ab})\br_{ab}\br_{ab}^{\top}.
\end{equation}
Take any $a,b\in\{i\}\cup c$ and any direction $\bu\in\mathcal S_i^\perp$.
Since $\br_{ab}\in\mathcal S_i$,
\[
\cov_{\bftheta}\left(
\bu^\top\tilde\nabla f(\tilde\bx_a),
f(\bx_b)
\right)
=
\bu^\top
\nabla_{\tilde\bx_a}k_{\bftheta}(\bx_a,\bx_b)
=
0.
\]
The same equality holds with $f(\bx_b)$ replaced by $y_b$, because the function observation noise is independent of the Gaussian process and of the gradients. 

Next take any $\bv\in\mathcal S_i$.
Using \eqref{eq:mixed_derivative_identity},
\[
\cov_{\bftheta}\left(
\bu^\top\tilde\nabla f(\tilde\bx_a),
\bv^\top\tilde\nabla f(\tilde\bx_b)
\right)
=
\bu^\top
\left[
-2\kappa_{\bftheta}'(r_{ab})\bI_d
-
4\kappa_{\bftheta}''(r_{ab})\br_{ab}\br_{ab}^{\top}
\right]
\bv
=
0.
\]
The identity term vanishes because $\bu\perp\bv$.
The rank-one term vanishes because $\bu\perp\br_{ab}$.

Decompose each scaled reference gradient into its components in $\mathcal S_i$ and $\mathcal S_i^\perp$.
Let $\tilde\bg_{c,\parallel}$ collect the components in $\mathcal S_i$, and let $\tilde\bg_{c,\perp}$ collect the components in $\mathcal S_i^\perp$.
The covariance calculations above imply
\[
\cov_{\bftheta}\left(
\tilde\bg_{c,\perp},
(y_i,\by_c,\tilde\bg_{c,\parallel})
\right)
=
\bfzero.
\]
Since all variables are jointly Gaussian,
\[
p_{\bftheta}
\left(
y_i\mid \by_c,\tilde\bg_{c,\parallel},\tilde\bg_{c,\perp}
\right)
=
p_{\bftheta}
\left(
y_i\mid \by_c,\tilde\bg_{c,\parallel}
\right).
\]

It remains to show that $\bq_c$ determines $\tilde\bg_{c,\parallel}$. For any scaled reference gradient $\tilde\bg_j$,
\[
\bfDelta_i^\top\tilde\bg_j
=
\bD_i^\top\bg_j.
\]
The map $\bv\mapsto \bfDelta_i^\top\bv$ is one-to-one on $\mathcal S_i$.
Indeed, if $\bv\in\mathcal S_i$ and $\bfDelta_i^\top\bv=\bfzero$, then $\bv=\bfDelta_i\ba$ for some $\ba$ and
\[
\|\bv\|^2
=
\ba^\top\bfDelta_i^\top\bfDelta_i\ba
=
\ba^\top\bfDelta_i^\top\bv
=
0.
\]
Thus $\bv=\bfzero$.
Therefore
\[
\bq_c
=
(\bI_{m_c}\otimes\bD_i^\top)\bg_c
=
(\bI_{m_c}\otimes\bfDelta_i^\top)\tilde\bg_c
\]
determines $\tilde\bg_{c,\parallel}$, and conversely $\bq_c$ is a deterministic function of $\tilde\bg_{c,\parallel}$.
Thus conditioning on $\bq_c$ is equivalent to conditioning on $\tilde\bg_{c,\parallel}$.
Combining this equivalence with the conditional irrelevance of $\tilde\bg_{c,\perp}$ gives
\[
\begin{aligned}
p_{\bftheta}(y_i\mid\by_c,\bg_c)
&=
p_{\bftheta}(y_i\mid\by_c,\tilde\bg_{c,\parallel},\tilde\bg_{c,\perp})\\
&=
p_{\bftheta}(y_i\mid\by_c,\tilde\bg_{c,\parallel})\\
&=
p_{\bftheta}(y_i\mid\by_c,\bq_c).
\end{aligned}
\]

It remains to verify that the matched noise model preserves the same decomposition. 
Let $\mathcal S_i=\operatorname{col}(\bfDelta_i)$ in scaled input coordinates and write
$\tilde\bg_c=\tilde\bg_{c,\parallel}+\tilde\bg_{c,\perp}$ for the decomposition of the observed scaled gradients into components in $\mathcal S_i$ and $\mathcal S_i^\perp$.
For exact gradients, the proof above shows that $\tilde\bg_{c,\perp}$ is uncorrelated with $(y_i,\by_c,\tilde\bg_{c,\parallel})$.

Now suppose the gradient noise is scalar in scaled coordinates,
\[
\tilde\bfepsilon_j^g
=
\bfLambda_{\bftheta}^{-1/2}\bfepsilon_j^g
\sim
\normal(\bfzero,\sigma_g^2\bI_d),
\]
independently across inputs.
For any $\bu\in\mathcal S_i^\perp$ and $\bv\in\mathcal S_i$,
\[
\cov\left(
\bu^\top\tilde\bfepsilon_a^g,
\bv^\top\tilde\bfepsilon_b^g
\right)
=
\mathbbm{1}\{a=b\}\sigma_g^2\bu^\top\bv
=
0.
\]
Thus the noise also preserves the decomposition into $\mathcal S_i$ and $\mathcal S_i^\perp$.
Since the noise is independent of the GP and all variables are jointly Gaussian,
\[
\tilde\bg_{c,\perp}
\quad\text{is independent of}\quad
(y_i,\by_c,\tilde\bg_{c,\parallel}).
\]
The same argument as in the exact gradient case gives
\[
p_{\bftheta}(y_i\mid\by_c,\bg_c)
=
p_{\bftheta}(y_i\mid\by_c,\bq_c).
\]

The scaled noise condition above is equivalent to
$\bfepsilon_j^g\sim\normal(\bfzero,\sigma_g^2\bfLambda_{\bftheta})$
in the original gradient coordinates.
Thus ARD metrics are covered by the exactness result when the gradient noise covariance is matched to the metric.

By contrast, under iid unscaled gradient noise,
\[
\bfepsilon_j^g\sim\normal(\bfzero,\sigma_g^2\bI_d),
\qquad
\tilde\bfepsilon_j^g
\sim
\normal(\bfzero,\sigma_g^2\bfLambda_{\bftheta}^{-1}).
\]
This scaled covariance is scalar only when $\bfLambda_{\bftheta}\propto\bI_d$, apart from the exact gradient case $\sigma_g^2=0$.
Hence, for non-isotropic ARD metrics with iid unscaled gradient noise, the exact equality need not hold because the perpendicular noisy component can carry information about the parallel noisy component through the noise covariance.
The reduced noisy gradient still defines a valid Gaussian conditional density, but it is not generally identical to conditioning on the full noisy gradient in this case.
\end{proof}

\subsection{Proof of Lemma~\ref{lem:direct_blocks}: Explicit Construction of Reduced Conditional Blocks} 
\label{app:projected_blocks}

\begin{proof}
We provide the explicit algebraic formulas for the reduced blocks required in Lemma~\ref{lem:direct_blocks}. 
Fix a target input $\bx_i$ and an 
reference set $c=\{j_{1},\ldots,j_{m_c}\}$.
Let $\bH_i=\bD_i^\top\bfLambda_{\bftheta}\bD_i\in\bbR^{m_c\times m_c}$ be the target-centered Gram matrix from Section~\ref{sec:prelim_stationary}, and let $\be_a$ denote the $a$th Euclidean basis vector in $\bbR^{m_c}$.
Write, for $a,b=1,\ldots,m_c$,
\[
\begin{aligned}
\bh_a
&=
\bH_i\be_a
=
\bfDelta_i^\top(\tilde\bx_{j_{a}}-\tilde\bx_i)
\in\bbR^{m_c}, \\
r_{a}
&=
(\bx_i-\bx_{j_{a}})^\top
\bfLambda_{\bftheta}
(\bx_i-\bx_{j_{a}})
=
[\bH_i]_{aa},\\
r_{ab}
&=
(\bx_{j_{a}}-\bx_{j_{b}})^\top
\bfLambda_{\bftheta}
(\bx_{j_{a}}-\bx_{j_{b}})
=
[\bH_i]_{aa}+[\bH_i]_{bb}-2[\bH_i]_{ab}.
\end{aligned}
\]
Here $r_{a}$ and $r_{ab}$ are the stationary kernel inputs for target-reference and reference-reference pairs, respectively.
Let
\[
\bq_{a}
=
\bD_i^\top\bg_{j_a}
\in\bbR^{m_c},
\qquad
a=1,\ldots,m_c,
\]
so that $\bq_c = (\bq_{1}^\top,\ldots,\bq_{m_c}^\top)^\top$.
The required unconditioned reduced blocks are obtained directly as follows.
For a reference function value $y_{j_{b}}$,
\[
\bk_{\bq_{a} y_{j_{b}}}
=
2\kappa_{\bftheta}'(r_{ab})(\bh_a-\bh_b).
\]
For the target function value $y_i$,
\[
\bk_{\bq_{a} y_i}
=
2\kappa_{\bftheta}'(r_{a})\bh_a.
\]
For the covariance between the projected gradients at reference inputs $j_{a}$ and $j_{b}$,
\[
\bK_{\bq_{a}\bq_{b}}
=
-2\kappa_{\bftheta}'(r_{ab})\bH_i
-
4\kappa_{\bftheta}''(r_{ab})
(\bh_a-\bh_b)(\bh_a-\bh_b)^\top
+
\mathbbm{1}\{a=b\}\sigma_g^2\bD_i^\top\bD_i \in \bbR^{m_c \times m_c}.
\]
For exact gradients, the final noise term disappears. Under the scaled-coordinate noise model, the final term is instead $\mathbbm{1}\{a=b\}\sigma_g^2\bH_i$.
Stacking these blocks over $a,b=1,\ldots,m_c$ gives $\bK_{\bq_c\by_c}\in\bbR^{m_c^2\times m_c}$, $\bk_{\bq_c y_i}\in\bbR^{m_c^2}$, and $\bK_{\bq_c\bq_c}\in\bbR^{m_c^2\times m_c^2}$.

The projected mean is
\[
\bfmu_{\bq_c}
=
(\bI_{m_c}\otimes\bD_i^\top)\bfmu_{\bg_c},
\qquad
\bfmu_{\bg_c}
=
\left(
\nabla\mu_{\bftheta}(\bx_{j_{1}})^\top,
\ldots,
\nabla\mu_{\bftheta}(\bx_{j_{m_c}})^\top
\right)^\top.
\]

The conditional quantities used in evaluating the conditional density are ordinary Gaussian Schur complements with respect to $\by_c$:
\[
\bK_{\bq_c\bq_c\mid \by_c}
=
\bK_{\bq_c\bq_c}
-
\bK_{\bq_c\by_c}
\bK_{\by_c\by_c}^{-1}
\bK_{\by_c\bq_c},
\]
\[
\bk_{\bq_c y_i\mid \by_c}
=
\bk_{\bq_c y_i}
-
\bK_{\bq_c\by_c}
\bK_{\by_c\by_c}^{-1}
\bk_{\by_c y_i},
\]
and
\[
\bfmu_{\bq_c\mid \by_c}
=
\bfmu_{\bq_c}
+
\bK_{\bq_c\by_c}
\bK_{\by_c\by_c}^{-1}
(\by_c-\bfmu_{\by_c}).
\]
Here $\bK_{\by_c\by_c}\in\bbR^{m_c \times m_c}$ includes the function observation noise. These formulas provide all matrix and vector quantities needed for the corrections without materializing the $m_c d\times m_c d$ local derivative covariance.
\end{proof}

\subsection{Proof of Proposition~\ref{prop:vecchia_score}}
\label{app:vecchia_score}

\begin{proof}
Let
\[
\ell_i(\bftheta)
=
\log p_{\bftheta}\left(y_i\mid \by_{c(i)},\bq_{c(i)}\right),
\qquad
\tilde\ell_n(\bftheta)
=
\sum_{i=1}^n \ell_i(\bftheta).
\]
The ordering and conditioning sets are fixed.
Therefore $\bq_{c(i)}$ is a deterministic function of the observed gradients and inputs for each $i$.

It is enough to show that each local score has mean zero at $\bftheta_0$.
By the tower property,
\[
\bbE_{\bftheta_0}
\left[
\nabla_{\bftheta}\ell_i(\bftheta_0)
\right]
=
\bbE_{\bftheta_0}
\left[
\bbE_{\bftheta_0}
\left[
\nabla_{\bftheta}
\log p_{\bftheta}
\left(y_i\mid \by_{c(i)},\bq_{c(i)}\right)
\Big|_{\bftheta=\bftheta_0}
\mid
\by_{c(i)},\bq_{c(i)}
\right]
\right].
\]
Conditioning on $(\by_{c(i)},\bq_{c(i)})$, the inner expectation is
\[
\int
\nabla_{\bftheta}
\log p_{\bftheta}
\left(y_i\mid \by_{c(i)},\bq_{c(i)}\right)
\Big|_{\bftheta=\bftheta_0}
p_{\bftheta_0}
\left(y_i\mid \by_{c(i)},\bq_{c(i)}\right)
\,dy_i.
\]
Using
\[
\nabla_{\bftheta}
\log p_{\bftheta}
\left(y_i\mid \by_{c(i)},\bq_{c(i)}\right)
=
\frac{
\nabla_{\bftheta}
p_{\bftheta}
\left(y_i\mid \by_{c(i)},\bq_{c(i)}\right)
}{
p_{\bftheta}
\left(y_i\mid \by_{c(i)},\bq_{c(i)}\right)
},
\]
the inner expectation equals
\[
\int
\nabla_{\bftheta}
p_{\bftheta}
\left(y_i\mid \by_{c(i)},\bq_{c(i)}\right)
\Big|_{\bftheta=\bftheta_0}
\,dy_i.
\]
By the assumed interchange of differentiation and integration,
\[
\int
\nabla_{\bftheta}
p_{\bftheta}
\left(y_i\mid \by_{c(i)},\bq_{c(i)}\right)
\Big|_{\bftheta=\bftheta_0}
\,dy_i
=
\nabla_{\bftheta}
\int
p_{\bftheta}
\left(y_i\mid \by_{c(i)},\bq_{c(i)}\right)
\,dy_i
\Big|_{\bftheta=\bftheta_0}.
\]
The integral is one for every $\bftheta$, since the integrand is a conditional density in $y_i$.
Therefore the last term is zero, and hence
\[
\bbE_{\bftheta_0}
\left[
\nabla_{\bftheta}\ell_i(\bftheta_0)
\right]
=
\bfzero.
\]
Summing over $i$ gives
\[
\bbE_{\bftheta_0}
\left[
\nabla_{\bftheta}\tilde\ell_n(\bftheta_0)
\right]
=
\sum_{i=1}^n
\bbE_{\bftheta_0}
\left[
\nabla_{\bftheta}\ell_i(\bftheta_0)
\right]
=
\bfzero.
\]
This proves Proposition~\ref{prop:vecchia_score}.
\end{proof}

\section{Experimental Details}
\label{app:experiment_detail}

\subsection{Posterior Fidelity and Scaling}
For the dimension sweep, a fixed lengthscale would also change the effective prior smoothness.
For example, under a squared exponential (SE) covariance function, the correlation between two inputs $\bx$ and $\bx'$ is
\[
\rho_{\ell}^{\mathrm{SE}}(\bx,\bx')
=
\exp\left(
-\frac{\|\bx-\bx'\|^2}{2\ell^2}
\right).
\]
If $\ell$ is fixed while typical squared distances $\|\bx-\bx'\|^2$ grow with $d$, prior correlations decrease automatically.
The dimension sweep would then mix posterior fidelity with a change in the prior correlation structure caused by increasing dimension.
We therefore choose the lengthscale through a fixed median distance correlation criterion rather than treating it as a method tuning parameter.

Let $\{\bx_i\}_{i=1}^n$ denote the training inputs, generated from a Sobol sequence~\cite{Sobol1967OnTD}, and define
\[
r_{\mathrm{med}}
=
\operatorname{median}_{1\leq a < b \leq n}
\|\bx_a-\bx_b\| .
\]
For the Mat\'ern-5/2 covariance used in the simulations, the correlation between two inputs $\bx$ and $\bx'$ is
\[
\rho_{\ell}(\bx,\bx')
=
\left(
1+\sqrt{5}\frac{\|\bx-\bx'\|}{\ell}
+\frac{5\|\bx-\bx'\|^2}{3\ell^2}
\right)
\exp\left(
-\sqrt{5}\frac{\|\bx-\bx'\|}{\ell}
\right).
\]
The isotropic lengthscale $\ell>0$ is selected to satisfy
\[
\left(
1+\sqrt{5}\frac{r_{\mathrm{med}}}{\ell}
+\frac{5r_{\mathrm{med}}^2}{3\ell^2}
\right)
\exp\left(
-\sqrt{5}\frac{r_{\mathrm{med}}}{\ell}
\right)
=
\rho,
\]
where $\rho=0.3$ in the main experiments. For each run and sweep setting, the resulting $\ell$ is shared by all methods.

\subsection{Large-Scale Regression with Observed Gradients}

Let $\bR_i$ denote the Cartesian coordinate vector of the $i$th molecular configuration, let $E_i$ be its molecular energy, and let $\bF_i=-\nabla_{\bR}E_i$ be the force.
All methods use scaled inputs $\bx_i=\bR_i/3$.
Let $\bar E_{\mathrm{tr}}$ and $s_E$ denote the training energy mean and standard deviation.
For \tera and Exact GP, the observed function value and gradient are
\[
y_i
=
-\frac{E_i-\bar E_{\mathrm{tr}}}{s_E},
\qquad
\bg_i
=
\nabla_{\bx_i} y_i
=
\frac{3\bF_i}{s_E}.
\]
A predicted normalized function value $\widehat y_i$ is converted back to physical energy units by
\[
\widehat E_i
=
\bar E_{\mathrm{tr}}-s_E\widehat y_i .
\]

For DDSVGP and DSoftKI, following \cite{huang2026scaling}, energies and forces are normalized by a joint energy and force standard deviation $s_{EF}$.
Their observed function value and gradient are
\[
y_i^{EF}
=
-\frac{E_i-\bar E_{\mathrm{tr}}}{s_{EF}},
\qquad
\bg_i^{EF}
=
\nabla_{\bx_i} y_i^{EF}
=
\frac{3\bF_i}{s_{EF}}.
\]
A predicted normalized function value $\widehat y_i^{EF}$ is converted back by
\[
\widehat E_i
=
\bar E_{\mathrm{tr}}-s_{EF}\widehat y_i^{EF}.
\]
We compute raw energy RMSE after converting predictions back to physical energy units and use it as the main metric for fair comparison across methods. In our experiments, we report raw energy RMSE per atom to make errors comparable across molecular systems with different numbers of atoms. See Table~\ref{tab:md22_settings} for settings.

\begin{table}[tb]
\centering
\caption{
MD22 regression settings for the main experiments.
}
\label{tab:md22_settings}
\small
\setlength{\tabcolsep}{7pt}
\renewcommand{\arraystretch}{1.12}
\begin{tabular}{ll}
\toprule
\multicolumn{2}{l}{\text{Common settings}} \\
\midrule
Train/test split & $90/10$ \\
Seeds & $6535,\ 8830,\ 92357$ \\
Input scaling & Cartesian coordinates $\bR$ are scaled as $\bx=\bR/3$ \\
Kernel & Isotropic SE \\
Initial lengthscale & $1.0$ \\
Value noise parameter & $10^{-3}$ \\
Gradient noise parameter & $10^{-3}d$ for DSoftKI, $10^{-3}$ otherwise \\
Hyperparams optimization & Adam~\cite{kingma2015adam}, weight decay $0.0$\\
\midrule
\multicolumn{2}{l}{\text{Standardization before training}} \\
\midrule
Exact GP & Training energy standard deviation $s_E$ \\
\tera & Training energy standard deviation $s_E$ \\
DDSVGP & Joint energy-force training standard deviation $s_{EF}$ \\
DSoftKI & Joint energy-force training standard deviation $s_{EF}$ \\
\midrule
\multicolumn{2}{l}{\text{Method-specific optimization settings}} \\
\midrule
Exact GP & 50 full-batch  steps, lr $\in\{0.01,0.005\}$ \\
\tera & 1 epoch, conditioning set size 20, batch size $256$, lr $0.01$ \\
DDSVGP & 50 epochs, 2 inducing directions, dimension dependent batch size and lr \\
DSoftKI & 50 epochs, 512 interpolation points, dimension dependent batch size and lr \\
\midrule
\multicolumn{2}{l}{\text{Dimension dependent schedule for DDSVGP and DSoftKI}} \\
\midrule
$d<180$ & batch size $1024$, DDSVGP lr $0.012$, DSoftKI lr $0.008$ \\
$180\le d<300$ & batch size $512$, DDSVGP lr $0.006$, DSoftKI lr $0.004$ \\
$300\le d<1000$ & batch size $256$, DDSVGP lr $0.003$, DSoftKI lr $0.002$ \\
$d\ge1000$ & batch size $128$, DDSVGP lr $0.0015$, DSoftKI lr $0.001$ \\
\bottomrule
\end{tabular}
\end{table}

\subsection{High-Dimensional Bayesian Optimization}

We evaluate high-dimensional BO on Levy-200D and Ackley-200D over $[0,1]^d$.
All methods use the same initial designs, evaluation budgets, and random seeds.
Objective values are standardized before surrogate fitting. For \tera, observed gradients are divided by the same response standard deviation. All experiments are run in double precision on a single GPU.
Table~\ref{tab:bo_common_setup} summarizes the common setup.
\begin{table}[t]
\centering
\small
\setlength{\tabcolsep}{7pt}
\renewcommand{\arraystretch}{1.15}
\caption{Common setup for the high-dimensional BO experiments.}
\label{tab:bo_common_setup}
\begin{tabular}{ll}
\toprule
\text{Setting} & \text{Value} \\
\midrule
Benchmarks & Levy-200D, Ackley-200D \\
Domain & $[0,1]^d$ \\
Evaluation budget & $150$ objective evaluations \\
Initial design & $30$ Sobol points \\
Batch size & $1$ \\
Random seeds & $1,27,42,86,99$ \\
Precision & double \\
\bottomrule
\end{tabular}
\end{table}
VBO follows the implementation of \cite{Hvarfner2024VanillaBO}, with an ARD SE GP and a dimension-scaled log-normal lengthscale prior. TuRBO-1 uses a local ARD Mat\'ern-5/2 GP inside the TuRBO-1 trust region and optimizes LogEI within the current trust region.
\tera uses an ARD Mat\'ern-5/2 derivative GP surrogate with $m=20$ local conditioning inputs and target-specific reduced gradients.
The function-value and gradient observation-noise variances in the \tera surrogate are fixed at $10^{-3}$ for numerical stability.
Table~\ref{tab:bo_method_setup} gives the method-specific settings.
\begin{table}[t]
\centering
\small
\setlength{\tabcolsep}{4.5pt}
\renewcommand{\arraystretch}{1.18}
\caption{Surrogate fitting and acquisition optimization settings for the main experiments.}
\label{tab:bo_method_setup}
\begin{tabularx}{\linewidth}{lYYY}
\toprule
\text{Method}
& \text{Surrogate}
& \text{Model fitting}
& \text{Acquisition optimization} \\
\midrule
VBO
& ARD SE GP
& MAP estimation with a dimension-scaled lengthscale prior
& LogEI, $512$ raw candidates, $4$ restarts, up to $300$ L-BFGS iterations \\
\addlinespace[1pt]
TuRBO-1
& Local ARD Mat\'ern-5/2 GP
& $50$ Adam steps per BO iteration
& LogEI, $512$ raw candidates, $4$ restarts, up to $300$ L-BFGS iterations \\
\addlinespace[1pt]
\tera
& ARD Mat\'ern-5/2 derivative GP
& $50$ initial Adam steps, then $10$ Adam steps every $20$ BO iterations
& LogEI, $256$ raw candidates, $4$ restarts, up to $20$ L-BFGS iterations \\
\bottomrule
\end{tabularx}
\end{table}

For \tera, kernel hyperparameters are learned by maximizing the conditional Vecchia likelihood in \eqref{eq:conditional_vecchia_loglik}.
We run $50$ Adam steps after the initial Sobol design, followed by $10$ warm-started Adam steps every $20$ BO iterations.
Between these updates, the surrogate is rebuilt with the newly observed function values and gradients while the kernel hyperparameters are held fixed.
This separates the incorporation of new derivative observations from the re-estimation of the scaled local geometry.
In the BO setting, frequent hyperparameter updates can repeatedly change the metric used for neighbor selection and the reduced local conditionals, which in turn changes the acquisition landscape.
The periodic warm-started schedule keeps this geometry stable over short BO windows while still allowing the surrogate parameters to adapt as more data are collected.
During acquisition optimization, the candidate-specific neighbor set is treated as fixed within each local predictive evaluation, making the LogEI objective piecewise differentiable.

\section{Additional Experimental Results}
\label{app:additional_experiment}

This section provides additional empirical results that complement the main experiments. We first present further GP regression results with observed gradients, followed by additional high-dimensional BO results.

\subsection{Large-Scale Regression with Observed Gradients}
\label{app:md22_additional_experiment}

\paragraph{Isotropic SE Kernel.}

Figure~\ref{fig:md22_training_curves_rbf} shows that the main separation already appears during training. \tera starts below the other derivative baselines and continues to decrease, while DDSVGP improves slowly and DSoftKI eventually deteriorates after an initial decrease. One possible explanation is that DDSVGP and DSoftKI optimize substantially larger global objectives, including inducing or interpolation parameters in addition to kernel and noise parameters. These additional global parameters can make likelihood or ELBO optimization difficult in high-dimensional spaces. Thus, the figure shows that the performance gap is visible not only in the final errors, but also in the observed training dynamics.
Table~\ref{tab:md22_raw_energy_rmse} gives the final test errors across all six systems. \tera is the best method on every dataset, with raw RMSE per atom between $0.002$ and $0.011$. Exact GP is consistently worse than \tera on every dataset where it is run. DDSVGP and DSoftKI use derivative information, but they remain substantially less accurate than \tera, and in several cases even less accurate than Exact GP, especially on the larger molecular systems.

\begin{figure}[tb]
    \centering
    \includegraphics[width=0.99\textwidth]{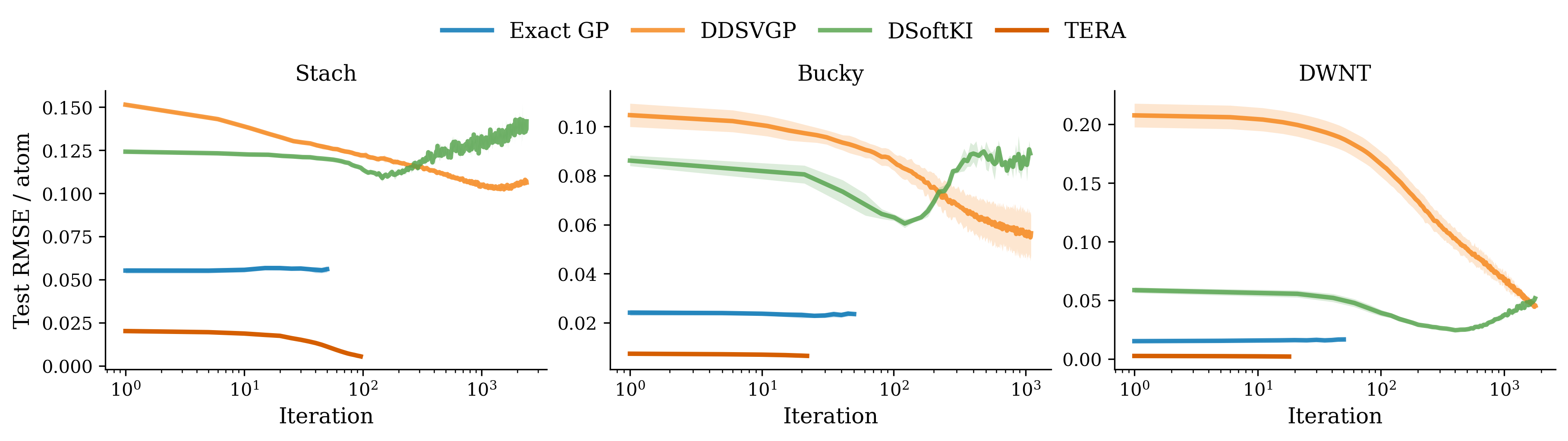}
    \caption{Training curves on selected MD22 benchmarks using an isotropic SE kernel. Each panel reports test RMSE per atom as a function of training step. \tera reaches the lowest error within its shorter training schedule.}
    \label{fig:md22_training_curves_rbf}
\end{figure}

\begin{table*}[tb]
    \caption{MD22 test energy prediction using an SE kernel with a single isotropic lengthscale for six benchmarks. Entries report raw energy RMSE per atom. Parentheses under each dataset give $(n,d)$, where $n$ is the number of molecular configurations and $d$ is the input dimension.}
    \small
    \centering
    \setlength{\tabcolsep}{4.0pt}
    \begin{tabular}{lrrrrrr}
         \toprule
         & DHA
         & AT-AT
         & Stachyose
         & AT-AT-CG-CG
         & Buckyball
         & DWNT \\
         & {\footnotesize $(62777,168)$}
         & {\footnotesize $(18000,180)$}
         & {\footnotesize $(24544,261)$}
         & {\footnotesize $(9137,354)$}
         & {\footnotesize $(5491,444)$}
         & {\footnotesize $(4528,1110)$} \\
         \midrule
         Exact GP
         & --
         & 0.054 {\footnotesize $\pm$ 0.001}
         & 0.056 {\footnotesize $\pm$ 0.002}
         & 0.042 {\footnotesize $\pm$ 0.001}
         & 0.024 {\footnotesize $\pm$ 0.000}
         & 0.017 {\footnotesize $\pm$ 0.001} \\
         DDSVGP
         & 0.155 {\footnotesize $\pm$ 0.001}
         & 0.119 {\footnotesize $\pm$ 0.003}
         & 0.107 {\footnotesize $\pm$ 0.001}
         & 0.068 {\footnotesize $\pm$ 0.005}
         & 0.056 {\footnotesize $\pm$ 0.009}
         & 0.045 {\footnotesize $\pm$ 0.003} \\
         DSoftKI
         & --
         & 0.170 {\footnotesize $\pm$ 0.004}
         & 0.139 {\footnotesize $\pm$ 0.005}
         & 0.126 {\footnotesize $\pm$ 0.003}
         & 0.089 {\footnotesize $\pm$ 0.005}
         & 0.048 {\footnotesize $\pm$ 0.003} \\
         \tera
         & \textbf{0.004} {\footnotesize $\pm$ 0.001}
         & \textbf{0.009} {\footnotesize $\pm$ 0.000}
         & \textbf{0.005} {\footnotesize $\pm$ 0.000}
         & \textbf{0.011} {\footnotesize $\pm$ 0.000}
         & \textbf{0.007} {\footnotesize $\pm$ 0.000}
         & \textbf{0.002} {\footnotesize $\pm$ 0.000} \\
         \bottomrule
    \end{tabular}
    \vspace{2pt}
    \begin{minipage}{0.95\linewidth}
    \footnotesize
    Buckyball denotes buckyball-catcher, and DWNT denotes double-walled-nanotube. Dashes indicate settings that were infeasible due to the resource constraint.
    \end{minipage}
    \label{tab:md22_raw_energy_rmse}
\end{table*}

\paragraph{ARD Mat\'ern-5/2 Kernel.}

As shown in Figure~\ref{fig:md22_main_metrics_matern}, \tera achieves the lowest test RMSE per atom on every benchmark using an ARD Mat\'ern-$5/2$ kernel.
The results show that, under this more flexible kernel, \tera retains the same overall pattern as in the isotropic SE kernel setting.
It gives the best predictive accuracy while keeping computational cost practical.
Figure~\ref{fig:md22_training_curves_matern} shows the corresponding training curves.

\begin{figure}[tb]
    \centering
    \includegraphics[width=0.99\textwidth]{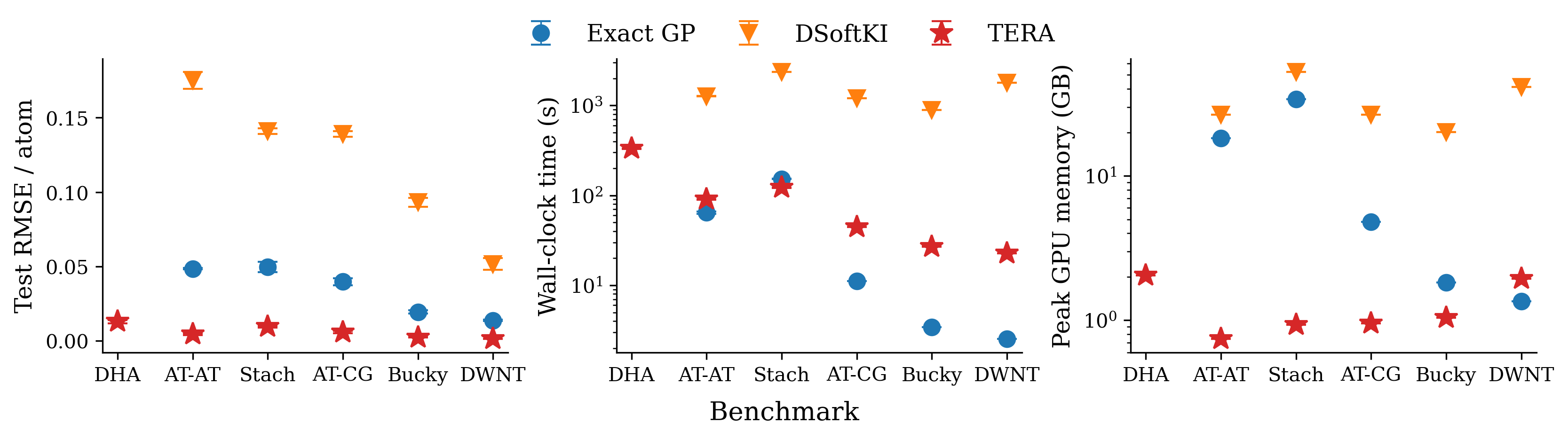}
    \caption{MD22 results with an ARD Mat\'ern-5/2 kernel. Panels show raw energy test RMSE per atom, end-to-end wall-clock time, and peak GPU memory. \tera achieves the lowest test error on every benchmark. DDSVGP is omitted because the implementation does not support ARD Mat\'ern kernels.}
    \label{fig:md22_main_metrics_matern}
\end{figure}

\begin{figure}[tb]
    \centering
    \includegraphics[width=0.99\textwidth]{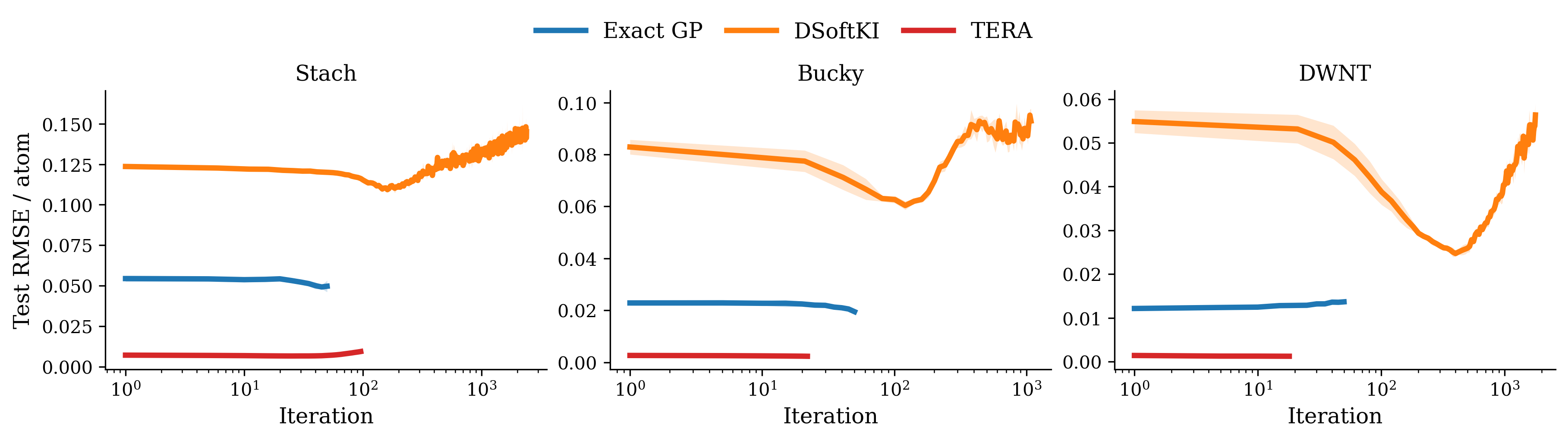}
    \caption{Training curves on selected MD22 datasets using an ARD Mat\'ern-5/2 kernel. Each panel presents test RMSE per atom as a function of training step. \tera reaches the lowest error within its shorter training schedule.}
    \label{fig:md22_training_curves_matern}
\end{figure}

\subsection{High-Dimensional Bayesian Optimization}
\label{app:bo_additional}

The main BO experiment compares all methods under the same small evaluation budget.
However, the original TuRBO-1 evaluation on Ackley-200D uses a much larger query budget and Thompson sampling (TS)~\cite{Thompson1933ONTL} inside the trust region.
We therefore include an additional comparison against TuRBO-1 TS in this original large-budget setup.
This experiment tests whether a strong value-only local BO method can close the gap when it is allowed orders of magnitude more function evaluations.

\begin{figure}[tb]
    \centering
    \includegraphics[width=0.99\textwidth]{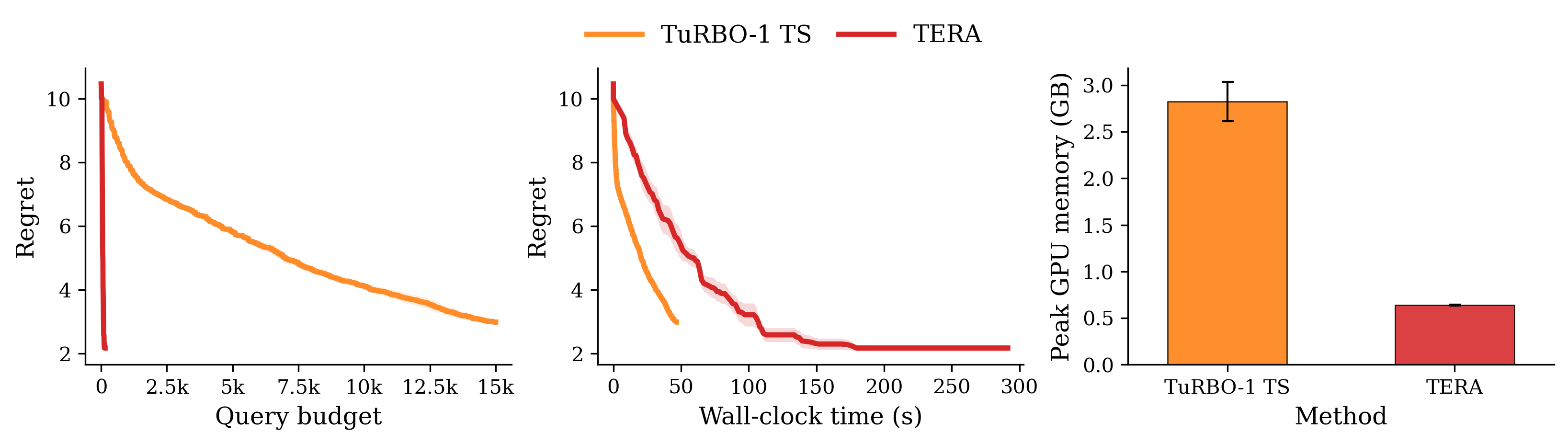}
    \caption{Additional comparison on Ackley-200D against TuRBO-1 TS in a large-budget setting. The left panel shows simple regret against the number of objective queries. The middle panel presents simple regret against end-to-end wall-clock time. The right panel reports peak GPU memory. TERA uses the same setting as in the main experiments.}
    \label{fig:bo_turbo_original}
\end{figure}

Figure~\ref{fig:bo_turbo_original} shows that the larger query budget substantially improves TuRBO-1 TS, but it does not eliminate the advantage of using gradients through \tera.
\tera reaches lower regret within its small budget, whereas TuRBO-1 TS requires thousands of function evaluations to approach the same range and remains above \tera at the end of the run.
The wall-clock comparison shows the expected tradeoff.
TuRBO-1 TS makes faster early progress in time because each update is cheap, while \tera spends more time per BO step evaluating candidate-specific Vecchia derivative GP factors.
Nevertheless, \tera attains the lower final regret.
The memory panel shows that \tera uses far less peak GPU memory than TuRBO-1 TS in this experiment.
Overall, this result indicates that the improvement in the main BO experiment is not merely an artifact of using a small query budget for TuRBO.
On Ackley-200D, \tera's derivative GP surrogate yields better optimization accuracy even when the TuRBO baseline is run with a much larger evaluation budget.

\section{Ablation Study}
\label{app:ablation}

\subsection{Sensitivity to Kernel Smoothness}
\label{app:target_corr_ablation}

\begin{figure}[tb]
    \centering
    \includegraphics[width=0.99\textwidth]{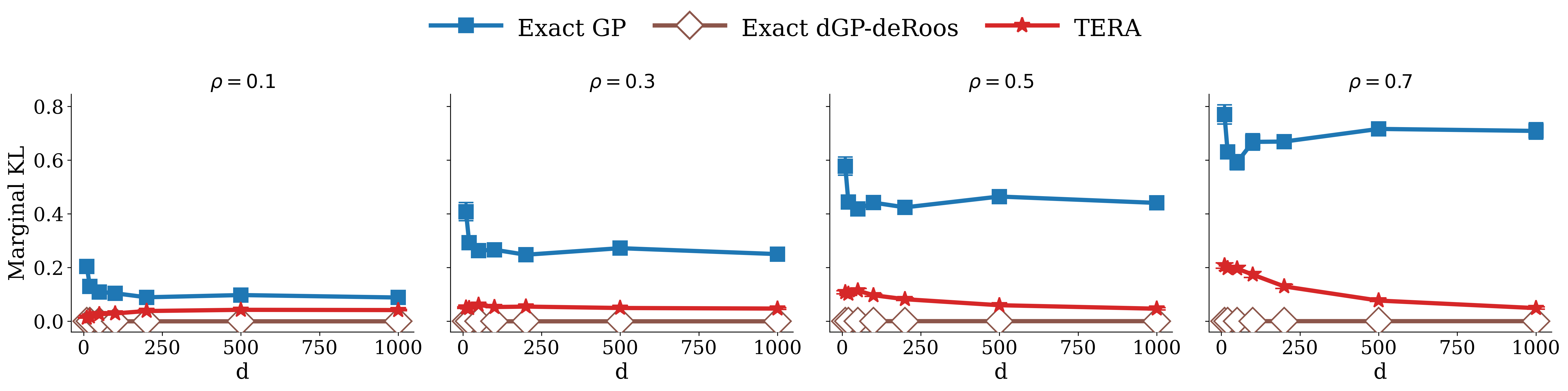}
    \caption{
    Sensitivity to the target median prior correlation in the GP simulation $d$ sweep.
    For each panel, the Mat\'ern-5/2 lengthscale is calibrated so that two training inputs at the median pairwise distance have prior correlation $\rho$.
    With $m=20$ fixed, we vary $\rho\in\{0.1,0.3,0.5,0.7\}$ and report average marginal KL divergence to the exact derivative GP posterior as the input dimension $d$ increases.
    The setting $\rho=0.3$ is used in the main experiments.
    }
    \label{fig:target_corr_ablation}
\end{figure}

Figure~\ref{fig:target_corr_ablation} shows how the posterior fidelity results vary with the correlation level used to calibrate the simulated GP prior.
Across all values of $\rho$, Exact dGP-deRoos remains essentially at zero KL, confirming that it recovers the exact derivative GP reference.
Exact GP has a substantially larger discrepancy, especially for smoother priors with larger $\rho$, reflecting the predictive information lost when gradient observations are ignored.
In contrast, \tera remains much closer to the exact derivative GP posterior across the tested dimensions.
The results confirm that its posterior fidelity is stable across both rougher and smoother simulated priors.

\subsection{Effect of Vecchia Conditioning Set Size}
\label{app:effect_m}

\begin{figure}[tb]
    \centering
    \includegraphics[width=0.99\textwidth]{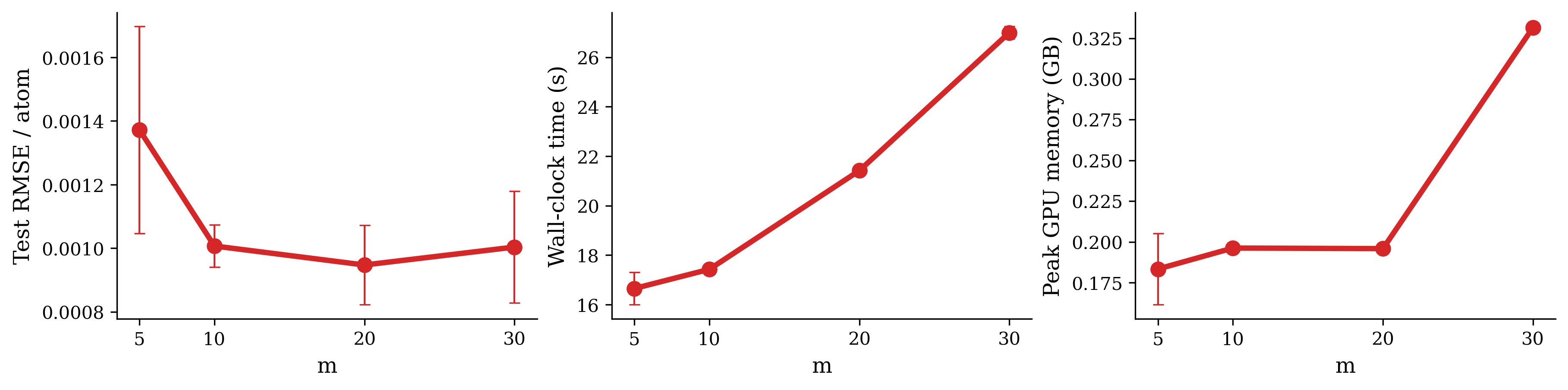}
    \caption{Effect of the Vecchia conditioning set size $m$ on the Buckyball-catcher benchmark.
    The panels report raw energy test RMSE per atom, end-to-end wall-clock time, and peak GPU memory for \tera with an isotropic Mat\'ern-5/2 kernel. Error bars show standard errors across runs.}
    \label{fig:md22_ablation_m_sweep}
\end{figure}

Figure~\ref{fig:md22_ablation_m_sweep} evaluates the tradeoff induced by the Vecchia conditioning set size $m$ on the GP regression experiment.
Increasing $m$ gives each local conditional access to a larger neighborhood, which can improve the approximation quality of the conditional Vecchia surrogate.
At the same time, the reduced local covariance blocks have size at most $m^2 \times m^2$, so the cost of the dense local solve grows quickly with $m$.
The left panel shows that very small neighborhoods are less accurate. Moving from $m=5$ to $m=10$ gives a clear reduction in test RMSE, while increasing to $m=20$ gives the best value in this sweep.
The improvement saturates after this point, and $m=30$ does not further improve prediction accuracy.
The middle and right panels show the corresponding computational tradeoff.
Wall-clock time increases monotonically with $m$, and peak GPU memory remains modest for $m\leq 20$ but increases at $m=30$.
These results support the default choice $m=20$ used in the main GP regression experiments.

\begin{figure}[tb]
    \centering
    \includegraphics[width=0.99\textwidth]{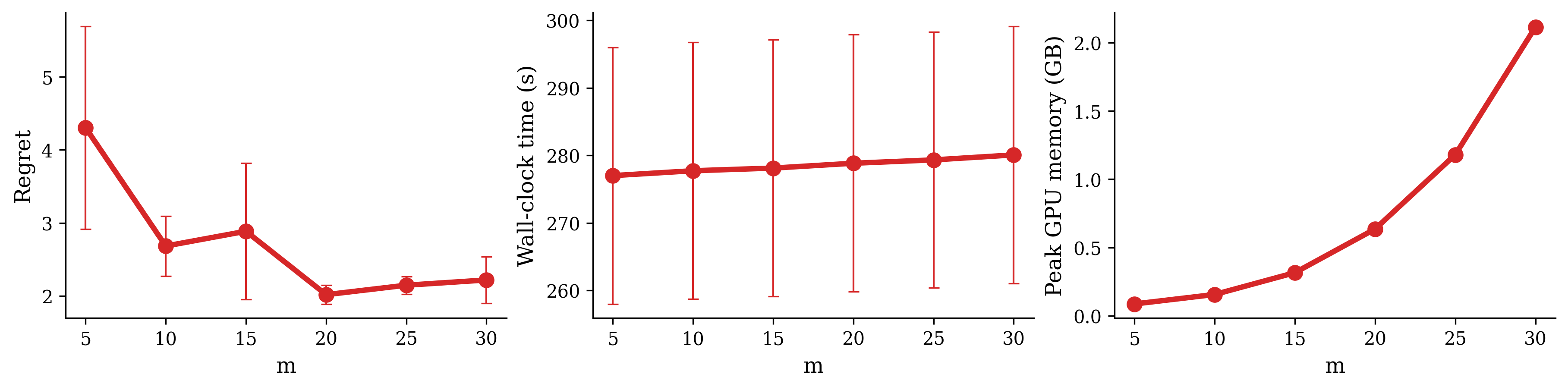}
    \caption{Effect of the Vecchia conditioning set size $m$ on the Ackley-200D benchmark.
    The panels show simple regret, end-to-end wall-clock time, and peak GPU memory for \tera with an ARD Mat\'ern-5/2 kernel. Error bars show standard errors across runs.}
    \label{fig:bo_ablation_m_sweep}
\end{figure}

Figure~\ref{fig:bo_ablation_m_sweep} shows the BO conditioning size sweep on Ackley-200D.
The left panel presents that very small neighborhoods are less effective, with $m=5$ giving the largest mean regret.
Increasing to $m=10$ and $m=20$ improves optimization performance, and the lowest mean regret in this sweep is obtained at $m=20$.
Further increasing the conditioning size to $m=25$ or $m=30$ does not yield an additional improvement under the same query budget.
Increasing $m$ gives each local conditional a larger conditioning set, but the observed BO regret also depends on hyperparameter learning, acquisition optimization, and the sequence of selected points.
The computational panels show a clearer tradeoff.
Wall-clock time changes only mildly across this sweep, because the BO runtime in this setting is dominated by acquisition optimization rather than by the size of $m$ alone.
Peak GPU memory, in contrast, increases sharply with $m$, reflecting the growth of the reduced local covariance blocks.

\subsection{Sensitivity to Observation Noise}

\begin{figure}[tb]
    \centering
    \includegraphics[width=0.45\textwidth]{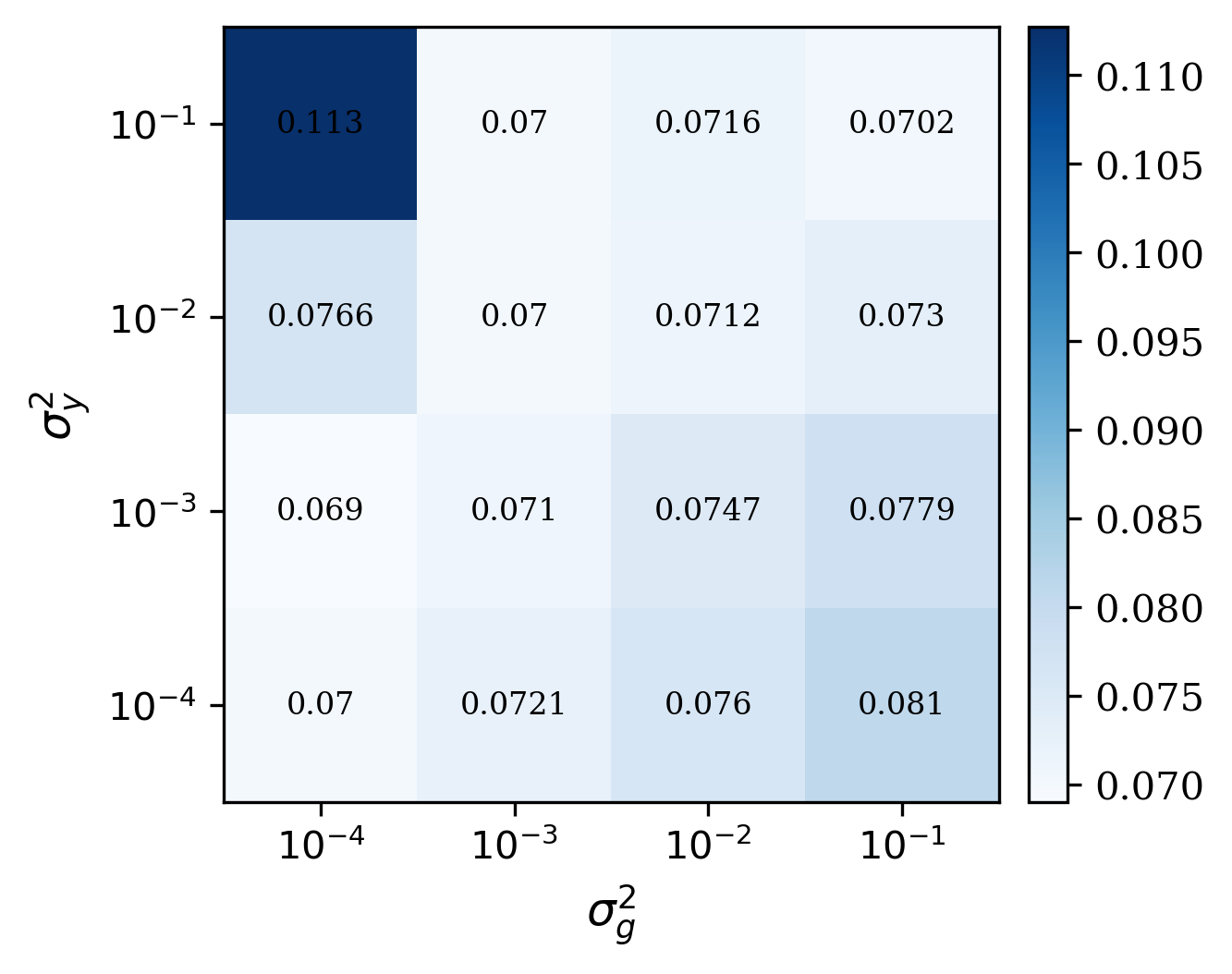}
    \caption{Sensitivity of \tera to the initialization of the observation noise variances for function values, $\sigma_y^2$, and gradients, $\sigma_g^2$, on the Buckyball-catcher benchmark. 
    Each cell reports raw test RMSE per atom for an isotropic Mat\'ern-5/2 kernel. Lower values are better.}
    \label{fig:md22_ablation_noise}
\end{figure}

Since the main GP regression experiments initialize the observation-noise parameters to fixed values, we also test the sensitivity of \tera to these initializations.
Figure~\ref{fig:md22_ablation_noise} shows how the initial values of the observation noise variances for function values, $\sigma_y^2$, and gradients, $\sigma_g^2$, affect test accuracy.
The heatmap varies $\sigma_y^2$ and $\sigma_g^2$ over four orders of magnitude and reports the resulting test RMSE per atom.
Except for the combination of large function value noise and very small gradient noise, the RMSE values remain in a narrow range across the grid.
The main failure mode appears when $\sigma_y^2$ is initialized too large while $\sigma_g^2$ is initialized too small.
This setting overweights gradient information relative to function values at the start of training and gives the largest error in the grid.
For moderate or small $\sigma_y^2$, the method is substantially less sensitive to $\sigma_g^2$.
Overall, the ablation indicates that \tera does not require a finely tuned noise initialization, but it benefits from avoiding extreme initial imbalances between $\sigma_y^2$ and $\sigma_g^2$.

\clearpage



\end{document}